%% file: FAE_arxiv.tex
\documentclass[a4paper, 12pt]{amsart}

\usepackage[colorlinks]{hyperref}

\usepackage{microtype}
\usepackage{graphicx}
\usepackage{subfigure}
\usepackage{booktabs} 

\usepackage{amsmath,amsfonts,amssymb,amsopn,amscd,amsthm}
\usepackage{comment}
\usepackage{dsfont}
\usepackage{graphicx}
\usepackage{color}
\usepackage[left=3.5cm,top=2.5cm,bottom=2cm,right=3cm]{geometry}
\usepackage{tikz-cd}

\let\oldtocsection=\tocsection
 
\let\oldtocsubsection=\tocsubsection
 
\let\oldtocsubsubsection=\tocsubsubsection
 
\renewcommand{\tocsection}[2]{\hspace{0em}\oldtocsection{#1}{#2}}
\renewcommand{\tocsubsection}[2]{\hspace{2em}\oldtocsubsection{#1}{#2}}
\renewcommand{\tocsubsubsection}[2]{\hspace{4em}\oldtocsubsubsection{#1}{#2}}

\usepackage{amsmath,amsfonts,amssymb,amsopn,amscd,amsthm}
\usepackage{stmaryrd}
\usepackage{tikz}
\usepackage{adjustbox}
\usepackage{tikz-cd}
\usepackage{dsfont}
\usepackage{algorithm, algorithmic}

\usepackage[foot]{amsaddr}

\title[Fibered Auto-Encoders]{Geodesics in fibered latent spaces: A geometric approach to learning correspondences between conditions}

\date{}

\author[Daouda]{Tariq \textsc{Daouda}*}
\author[Chhaibi]{Reda \textsc{Chhaibi}*}
\author[Tossou]{Prudencio \textsc{Tossou}}
\author[Villani]{Alexandra-Chlo\'e \textsc{Villani}}

\address{*: Equal contribution}
\address[Daouda and Villani]{Massachusetts General Hospital, Boston, MA \& Harvard Medical School, Boston, MA \& Broad Institute, Cambridge, MA}
\address[Chhaibi]{Universit\'e Paul Sabatier, Toulouse 3 -- Institut de math\'ematiques de Toulouse (IMT) -- 118, route de Narbonne, 31400, Toulouse, France}
\address[Tossou]{InvivoAI, Montreal, QC \& Universit\'e Laval, Qu\'ebec, QC}

\email{tdaouda@broadinstitute.org, reda.chhaibi@math.univ-toulouse.fr}
\email{prudencio@invivoai.com,
avillani@mgh.harvard.edu}

\allowdisplaybreaks[4]

\def\half{\frac{1}{2}}

\def\1{{\mathbf 1}}

\DeclareMathOperator{\Span}{Span}
\DeclareMathOperator{\id}{id}

\def\N{{\mathbb N}}

\def\R{{\mathbb R}}

\def\P{{\mathbb P}}
\def\E{{\mathbb E}}

\def\Cc{{\mathcal C}}
\def\Dc{{\mathcal D}}
\def\Ec{{\mathcal E}}

\def\Lc{{\mathcal L}}

\def\Nc{{\mathcal N}}

\def\Xc{{\mathcal X}}

\newtheorem{thm}{Theorem}[section]

\newtheorem{questions}[thm]{Questions}

\newtheorem{lemma}[thm]{Lemma}

\numberwithin{figure}{section}

\numberwithin{algorithm}{section}
\numberwithin{equation}{section}

\begin{document}

\begin{abstract}
This work introduces a geometric framework and a novel network architecture for creating correspondences between samples of different conditions. Under this formalism, the latent space is a fiber bundle stratified into a \textit{base space} encoding conditions, and a \textit{fiber space} encoding the variations within conditions. Furthermore, this latent space is endowed with a natural pull-back metric.
The correspondences between conditions are obtained by minimizing an energy functional, resulting in diffeomorphism flows between fibers.

We illustrate this approach using MNIST and Olivetti and benchmark its performances on the task of batch correction, which is the problem of integrating multiple biological datasets together.
\end{abstract}

\maketitle

\medskip


\clearpage
\tableofcontents
\clearpage

\section{Introduction}
\label{section:introduction}

Most datasets can be naturally stratified by their meta-data and annotations. For example, biological datasets could be labelled by patient age, ethnicity and medical history. Bank transactions by gender, time periods, locations. Conceptually, these annotations are considered to be generative of the variations observed in the dataset. They are intuitively understood as containing \emph{explanatory factors of variation} that can be independently used to generate the dataset samples. Following this intuition, modelling therefore needs to account for sample differences by explicitly disentangling these factors of variation from the labels themselves. This disentanglement is critical for many practical applications such as: understanding the factors of variation within the dataset and understanding the correspondences between samples from different datasets of similar nature.

In the context of generative models with a latent space, this brings forth the following questions:
\begin{questions}
\label{questions}
\begin{itemize}
    \item[(a)] How to structure the latent space in order to separate conditions from factors of variation within a condition? 
    \item[(b)] How to quantitatively evaluate the closeness of two conditions and transport data points from one to another?
    \item[(c)] An intuitive definition of disentanglement states that, ideally, conditions should be independent from variations within a condition. How to enforce this desirable feature?
\end{itemize}
\end{questions}

In this work, we propose a geometric formalism that accounts for dependencies between conditions and latent variables with a supervised geometric disentanglement directly implemented at the level of the latent space and the learning. The natural idea is to take a Riemannian fiber bundle $M$ for latent space. That is to say that $M$ is stratified into a base space $B$, encoding conditions, and a fiber space $F$, encoding variations within conditions. Within this formalism samples are encoded using two coordinates $\left( f, b \right) \in B \times F$. This answers question (a) among Questions \ref{questions}, at least a surface level. Much more will be said throughout the paper.

More importantly, using a natural Riemannian structure on $M$, this formalism is made quantitative and allows us to compute correspondences between conditions. We answer question (b) among Questions \ref{questions} by reducing the problem of translating samples between two conditions to finding the geodesics (shortest paths) in $M$ linking their corresponding fibers. We achieve this by minimizing an energy functional, which allow us to calculate diffeomorphism flows between fibers (see a conceptual sketch in Fig. \ref{fig:concept}). 

\begin{figure}[htp!]
\begin{center}
\includegraphics[width=0.6\textwidth]{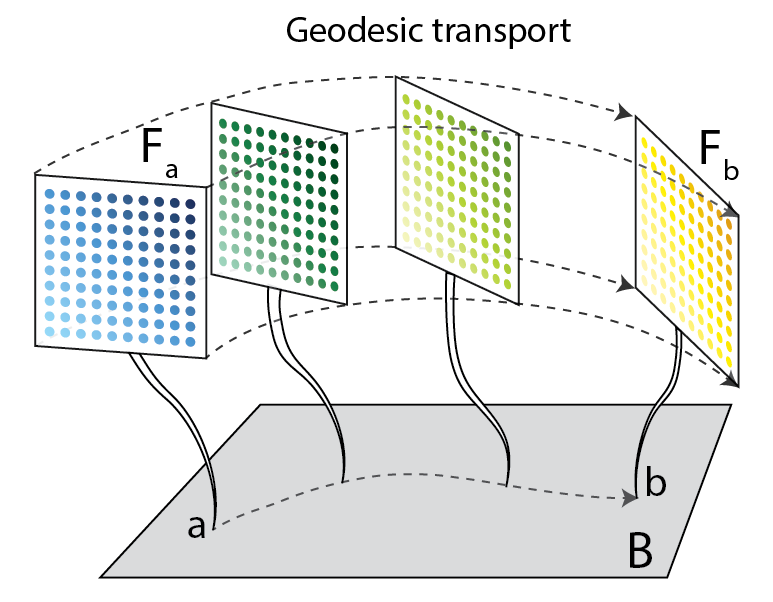}
\caption{Conceptual sketch of the proposed method. We stratify the learned latent space $M$ into a base space $B$ and a fiber space $F$. Under this representation, samples can be formally transported form $F_a$ to $F_b$, and geodesic interpolations between latent spaces $F_a$ and $F_b$ can be generated.}
\label{fig:concept}
\end{center}
\end{figure}

Finally, to answer question (c) among Questions \ref{questions}, we perform an explicit geometric disentanglement. It is baked into our neural network architecture which constructs the latent space while respecting the formalism.

\subsection{Literature review}
Works on disentanglement have concentrated on implicitly disentangling these factors of variation using unsupervised learning approaches \cite{SCAN, Hauberg, BetaVAE, FactorVAE, BetaTCVAE, DIPVAE, FAIRVAE}. However, recently it was demonstrated that disentanglement cannot be achieved without some from of inductive biases at both the level of the learning algorithm and the dataset \cite{locatello2019challenging}.

Current conditioned methods (methods that mix supervised labelling with unsupervised learning) often fail to consider or neglect the disentanglement between labels and factors of variations. In fact, the distribution of a variable $x$ stratified by a condition $c$ is often modelled using latent variables $z$ as $p(x| z, c)$ without any considerations for the dependencies between $z$ and $c$. For example, variational approaches such as Conditional Variational Auto-Encoders (CVAEs) \cite{CVAE, CVAE2, TRVAE} typically use bi-modal decoders that take both the latent variables and the condition encodings but do not account for duplication in information between them. Similar architectures are used for Conditional Generative Adversarial Networks (CGAN) \cite{CGAN, SSGAN, CYCLE_GAN} and hybrid methods mixing GANs and VAE \cite{AAN, VAE_GAN}. The recent paper \cite{Hauberg} on the other hand explores models with explicit disentanglement methods in a VAE framework. The model is however not conditioned and relies instead on a ST-Layer to learn deformations of latent space. 

\subsection{Contributions and structure of the paper}

Let us now move to a technical description of the paper's contributions and how their are structured.

\medskip

{\bf Section \ref{section:theory}: Theoretical contribution.} In that section, the idea presented in Fig. \ref{fig:concept} is put on a solid mathematical footing. This leads to our Main Theorem \ref{thm:main}, which establishes existence and generic uniqueness of the desired geodesics. Its proof in Subsection \ref{subsection:proof} can be skipped on a first read. Let us only mention that the existence follows the usual direct method in calculus of variations while generic uniqueness relies on the fine properties of the (generalized) cut-locus in Riemannian geometry.

Endowing latent spaces with a natural Riemannian geometry has already been explored in machine learning, especially in shape analysis. For a modern review, see \cite{PSF19} and the references therein. 
Our method for computing similarities bares some similarities with the work of \cite{GSFG12}. There, parallel transport is used in the context of vector bundles to transport features between datasets. Their Fig. 1 is in the spirit of what we propose. Nevertheless, our framework is fundamentally different as we do not assume that fibers carry a linear structure and we compute non-linear correspondences. The numerics of geodesics in a modern machine learning framework are explained in Subsection \ref{subsection:geodesics_numerics}.

\medskip

{\bf Section \ref{section:network}: The FAE architecture and explicit geometric disentanglement.} 
We implemented our geometric framework using a non-variational auto-encoder architecture, deemed Fibered Auto-Encoder (FAE). A descriptive sketch of the architecture is given in Fig. \ref{fig:nnArchitecture}. In fact, the FAE architecture is based on the preliminary work of \cite{HGN} and brings major improvements to it. One such improvement is the use of the standard method Domain Adversarial Training \cite{DANN1, DANN2} in order to explicitly disentangle the base space and the fiber space. We enforce an informational disentanglement which finalizes the geometric fiber structure of the latent space.

The method's success is attested by Fig. \ref{fig:MNIST_no_ap}, making fibers for different conditions indistinguishable: the condition-specific information within fibers does not allow to recover the underlying condition.

In principle, the approach can be adapted to other frameworks such as VAEs or GANs. We chose to focus on classical auto-encoders for their simplicity and versatility.

\medskip

{\bf Section \ref{section:Applications}: Applications.}
The effectiveness of both FAEs for stratified data modelling, and of the transport method are illustrated by experimenting on MNIST \cite{MNIST} and Olivetti \cite{SKLEARN} datasets. The MNIST dataset is particularly useful for visualizing geodesics and the obtained correspondences. For examples Figures \ref{fig:MNIST_diffeo}, \ref{fig:mnist_geodesics3d}, \ref{fig:MNIST_geodesics} and \ref{fig:mnist_diffeo_10} undoubtedly demonstrate that the natural geodesics in a latent space are {\it not} straight curves. They visually demonstrate that the approach is sound and that the Riemannian formalism is essentially unavoidable.

Finally, we give a real-life application  by quantifying the quality of transport by applying the proposed framework to the batch correction problem. Batch correction refers to the problem of multiple dataset integration in biology. It is a major hurdle preventing researchers from leveraging the power of previously published datasets. We use our approach to translate samples between datasets and provided benchmarks against state of the art methods, including other neural network approaches, and report overall on par, or better results.

\section{Geometric framework}
\label{section:theory}
We start by explaining how the problem of learning correspondences can formalized geometrically as geodesically transporting latent variables in a fiber bundle. The only assumption is that there exists a smooth map $\Psi_\theta$ such that for any coordinates $(f,b) \in F \times B$, one can generate samples, via:
$$ \widehat{X} = \Psi_\theta( f,b ) \ .$$
This function will be later constructed thanks to a neural network. To be exact, $\Psi_\theta$ will be the decoder part in our FAE architecture. And more generally, one may consider the corresponding map in any generative process.

\subsection{Fiber bundles}
Morally, a fiber bundle $M$ is a topological space locally stratified along two directions (see Fig. \ref{fig:fiberBundle}):
\begin{itemize}
\item The horizontal direction is tangent to a topological space called the base manifold $B$ (in our case encoding conditions).
\item The vertical direction is tangent to a topological space called the fiber manifold $F$ (in our case encoding variations within conditions).
\end{itemize}

\begin{figure}
\input{figures/figFiberBundle.tex}
\end{figure}

For example, representing the latent space of MNIST as a fiber bundle, the base manifold $B$ will be an Euclidean space where image labels are embedded. The fiber $F$ will be the latent space encoding digit variability such as boldness and inclination.

Formally, $M$ is a fiber bundle with base space $B$ and fiber $F$ when, $M$ is a topological space such that there exists a projection map $\pi: M \rightarrow B$ such that for every open set $U \subset B$, $\pi^{-1}(U)$ is diffeomorphic to $U \times F$ \cite{KSM99}.
The fiber sitting above $b \in B$ is denoted $F_b := \pi^{-1}(\{b\})$.

Here we are only be concerned with the simplest fiber bundle $M = B \times F$ where $\pi$ is the projection on the first coordinate. This space is known as the trivial fiber bundle, as it does not contain any global topological information. 

Nevertheless, in general, given $a \neq b \in B$, one should not put in correspondence $F_a = \{a\} \times F$ and $F_b = \{b\} \times F$ using the identity map on the second coordinate. We refer to this operation as \textbf{naive transport}, which is a valid approximation only under certain circumstances. For example, in the case of MNIST, this would amount to considering the spaces encoding the latent properties of $0$ and $1$ as strictly identical. A priori, spaces of latent properties do change upon changing classes. 

We formalize the problem of finding the best correspondence between fibers as a \textbf{geodesic transport} problem, which requires the Riemannian point of view developed in the next section \ref{subsection:geodesics}. 

\subsection{Geodesic transport}
\label{subsection:geodesics}
\begin{figure}
\input{figures/figGeodesicTransport.tex}
\end{figure}
In this section, we assume that the latent space $M$ is a general fiber bundle with standard fiber $F$ and base space $B$. The natural projection is written $\pi: M \rightarrow B$. Also we adopt the extrinsic point of view by assuming that $M$ is a submanifold of $\R^L$ for $L \in \N$ large enough\footnote{For the sake of simplicity, we keep the extrinsic point of view throughout the paper, instead of the modern intrinsic point view.}. By abuse of notation, elements in $M$ are written $(f,b) = (f, \pi(f))$ as if we are dealing with a product space. The target space $\Xc$ is Euclidean with norm denoted $\| \cdot \|_2$. It has a constant metric given by the usual scalar product $\langle \cdot, \cdot \rangle $. Our goal is to construct the shortest curves between fibers in $M$. This is illustrated in Fig. \ref{fig:transport}. Recall that we assume the existence of a generating process 
\begin{align}
\label{def:targ_space}
\Psi_\theta: & \ M \rightarrow \left( \Xc, \|\cdot \|_2 \right) \ .
\end{align}
The target space $\Xc$ is taken as Euclidean and endowed with the usual metric $\|\cdot \|_2$, although other metrics could be considered. The latent space $M$ however should not be seen as Euclidean. Since the Jacobian $\nabla \Psi_\theta$ measures the sensitivity with respect to latent variables, it is natural to use it to define a metric on $M$ where gradients are lower along shorter curves.

A strict mathematical reformulation consists in saying that the latent space $M$ is given the structure of a Riemannian manifold thanks to a metric tensor $g$, by pulling back the metric structure from the generating map $\Psi_\theta: (M, g) \rightarrow \left( \Xc, \|\cdot\|_2 \right)$. See \cite{GHL90}[$\mathsection 1.106$] for a formal definition of the pull-back operation. Thus the metric for the latent space $M$ respects the generative process rather than being the unnatural Euclidean metric on $M$.

\begin{lemma}
The pullback metric $g := \Psi_\theta^* \langle \cdot, \cdot \rangle$ at a point $p \in M$ identifies to the positive definite matrix
$$ g(p) = \nabla \Psi_\theta(p)^* \nabla \Psi_\theta(p) \ .$$
\end{lemma}
\begin{proof}
Let $\langle \cdot, \cdot \rangle_p^M$ and $\langle \cdot, \cdot \rangle_{\Psi_\theta(p)}^\Xc$ be the metrics at $p \in M$ and $\Psi_\theta(p) \in \Xc$. For any $p \in M$ and any two tangent vectors $(x,y) \in T_p M \times T_p M$, we have by definition of pullback for differential forms:
\begin{align*}
    \langle x, y \rangle_p^M
= & \langle \nabla \Psi_\theta(p) x,  \nabla \Psi_\theta(p) y \rangle_{\Psi_\theta(p)}^\Xc \\
= & \langle x, \nabla \Psi_\theta(p)^* \nabla \Psi_\theta(p) y \rangle \ ,
\end{align*}
where we used in the last line that $\Xc$ is Euclidean. It is indeed of the form  $\langle x, g(p) y \rangle \ .$
\end{proof}

Let $n = \dim B$ and $m = \dim F$. Given a point $p \in M$, we make the identification of the tangent space $T_p M \simeq \R^{m+n}$. Also the metric tensor is nothing but the datum of a positive definite matrix $g(p) \in M_{m+n}(\R)$ depending smoothly on the point $p \in M$. If $\langle \cdot, \cdot \rangle$ is the scalar product on $\R^{m+n}$, we obtain the Riemannian scalar product $\langle \cdot, \ g(p) \cdot \rangle$ on $T_p M \times T_p M$ for every point. For a curve $\gamma: [0,1] \rightarrow M$, we write $\gamma_t$ for the position and $\dot{\gamma}_t \in T_{\gamma_t} M$ for the speed vector at time $t$.

By definition, geodesics $\gamma$ minimize length
$$ L(\gamma)
 = \int_0^1 dt \sqrt{
   \langle \dot{\gamma}_t, g(\gamma_t) \dot{\gamma}_t \rangle
   } \ , 
$$
and upon reparametrization, it is well-known that \cite{GHL90}[$\mathsection 2.96$] geodesics also minimize the energy functional whose general expression is:
    \begin{align}
    \label{def:riemannianEnergy}
    \Ec(\gamma) := & \ \half \int_0^1 dt \ \langle \dot{\gamma}_t, 
    g(\gamma_t)
    \dot{\gamma}_t \rangle \ .
    \end{align}
The energy functional $\Ec$ has better convexity properties compared to the length $L$ and thus its minimization is better behaved numerically. 
In our case, we choose to define:
  \begin{align}
  \label{def:energy}
  \Ec(\gamma) := 
& \ \half \int_0^1 dt \ \left| \nabla \Psi_\theta \left( \gamma_t \right) \cdot \dot{\gamma}_t \right|^2 \ ,
  \end{align}
  which is of the same form as Eq. \eqref{def:riemannianEnergy} with
  \begin{align}
     \label{def:gTensor}
     g(p) = & \nabla \Psi_\theta \left( p \right)^* \nabla \Psi_\theta \left( p \right) \ .
  \end{align}
Here $*$ denotes the transposition applied to the Jacobian matrix $\nabla \Psi_\theta \left( m \right)$ . Upon comparing Eq. \eqref{def:riemannianEnergy} and Eq. \eqref{def:energy}, we see that minimizing the energy $\Ec(\gamma)$ with this metric tensor amounts exactly to choosing curves which are causing minimal changes in the output of the generator $\Psi_\theta$. Therefore, the role of neural network in the next section is not only to create a space of latent variables but also the geometry of a Riemannian manifold via the metric tensor \eqref{def:gTensor}.

In fact, considering the energy functional $\Ec$ for this metric amounts to the following equivalent formulation. First consider the energy functional on $\Xc$, whose geodesics are straight lines and which is given for any curve $c:[0,1] \rightarrow \Xc$ by:
$$ E(c) = \half \int_0^1 \| \dot{c}_t \|_2^2  \ .$$
Then restrict to curves the form $c_t = \Psi_\theta\left( \gamma_t \right)$ where $\gamma$ is a curve on $M$. This naturally yields:
\begin{align}
\label{eq:energy}
   \Ec(\gamma)
:= & \ \half \int_0^1 \| \nabla \Psi_\theta(\gamma_t) \dot{\gamma}_t \|^2_2 \\
 = & \ \half \int_0^1 \langle \dot{\gamma}_t, g(\gamma_t) \dot{\gamma}_t \rangle 
\nonumber
\ ,
\end{align}
which was our working definition. Basically, we are restricting the energy functional to the very small subspace $\textrm{Im} \Psi_\theta$ of admissible values generated by the neural network. Notice that after such a restriction, geodesics have no reason to be straight lines anymore.

In the end, the problem of creating correspondences is formalized as the following constrained minimization problem. Given a $(f_1, b_1) \in M$ and a point $b_2 \in B$, the best corresponding point $f_2 \in F_{b_2}$ is given by the endpoint of the curve $\gamma^*$ satisfying the constrained minimization problem:
  \begin{align}
  \label{def:variationalProblem}
    \gamma^{*}
:= & \underset{\substack{
      \gamma \in H^1\left( [0,1]; \left(M,g\right) \right), \\
      \gamma_{t=0} = (f_1, b_1), \\
      \gamma_{t=1} \in F_{b_2}
      }}
      {\textrm{Argmin}} \Ec(\gamma) \ ,
  \end{align}
where the Sobolev space $H^1\left( [0,1];  \left(M,g\right) \right)$ is the largest space on which energy is finite:
\begin{align}
\label{def:sobolev_on_M}
   H^1\left( [0,1]; \left(M,g\right) \right) 
:= & 
 \left\{ \gamma: [0,1] \rightarrow (M,g) \textrm{ measurable } \ | \ \Ec\left(\gamma\right) < \infty \right\} 
 \ .
\end{align}

The following theorem states that the problem \eqref{def:variationalProblem} is often well-posed:
\begin{thm}[Main theorem]
\label{thm:main}
Assume that $(M,g)$ is path-connected, complete, boundary free and that there exists a universal constant $C>0$ such that for all $p \in M$ and $v \in T_p M$:
\begin{align}
\tag{\bf H}
\label{comparison_hypothesis_appendix}
   \frac{1}{C} \|v\|_2^2 \ 
   \leq \ 
   \langle v, g(p) v \rangle \ 
   \leq \ 
   C \|v\|_2^2 \ .
\end{align}
Under such hypotheses, minimizing geodesics always exist, are smooth, have constant speed and are generically unique. "Generically" means for all $(f_1, b_1)$ outside a closed set of zero Lebesgue measure, the cut locus $\textrm{Cut}(F_{b_2})$, while $b_2$ is fixed.

Furthermore, the correspondence
\begin{align}
\label{def:correspondence}
    &
   \begin{array}{cccc}
   \Cc_{b_1}^{b_2}:
       & F_{b_1} \cap \textrm{Cut}(F_{b_2})^c & \longrightarrow & F_{b_2} \\
       & (f_1,b_1) = \gamma_{t=0}^*           & \mapsto         & (f_2,b_2) = \gamma_{t=1}^* \\
   \end{array}
\end{align}
is a well-defined local diffeomorphism between fibers, which tends to the identity map on $F$ as $b_1$ converges to $b_2$.
\end{thm}
The complete proof is available in the next subsection \ref{subsection:proof} and can be skipped on a first read. Subsection \ref{subsection:geodesics_numerics} shows how to numerically implement the minimization in a modern machine learning framework, and discusses the relevance of adding a regularization cost for high dimensional manifolds $M$. In applications, we will consider $B=\R^n$, $F=[-1,1]^m$ and ignore boundary effects. For now, let us discuss the setting and hypotheses, while setting a few notations.

\medskip

{\bf On the manifold $M$:} Assuming that $M$ is a path-connected and complete metric space is standard practice, as it avoids many geometric pathologies - see \cite{GHL90}[Corollary $\mathsection 2.105$]. Also we can assume general fiber bundles. Notice that unlike the setting of vector bundles or $G$-principal bundles, our setting has very little equivariance: there is no natural group action preserving the metric on fibers, and there is no natural quotient metric on $B$. As such, the convergence of $b_1 \rightarrow b_2$ so that $\Cc_{b_1}^{b_2} \rightarrow \id$ needs to be understood in the Euclidean topology of $M \subset \R^L$.

Regarding the setting of manifolds without boundaries, our arguments for existence and smoothness of geodesics carry verbatim to the case of manifolds with boundaries. However, the (already delicate) arguments proving generic uniqueness fail in that case. This is why we made the choice stating a complete theorem with $M$ having no boundary and we ignored boundary effects in applications.

\medskip

{\bf On the hypothesis (H):} Rather than a complete justification, let us explain why this hypothesis is acceptable in practice.

From the point of view of machine learning, the upper bound on the metric tensor 
\begin{align}
\label{eq:g_tensor}
g(p) = & \ \nabla \Psi_\theta(p)^* \nabla \Psi_\theta(p)     
\end{align}
is guaranteed on compact sets as long as $\Psi_\theta$ is smooth. The lower bound is more delicate. A necessary condition is the non-vanishing gradients (non-saturation) which we already guaranteed thanks to the choice of non-linearities detailed in subsection 3.1. 
A sufficient yet reasonable condition is to have $g(p)$ with full rank and $g(p)^{-1}$ uniformly bounded. We have to resort to a genericity argument to justify why this is reasonable: since $\Psi_\theta: M \rightarrow \Xc$ is a decoder with $\dim M$ much smaller than $\dim \Xc$,  then $\nabla \Psi_\theta(p)^* \nabla \Psi_\theta(p)$ has rank $\dim M$ for a generic Jacobian matrix $\nabla \Psi_\theta(p)$.

At the level of geometry, it ensures that the manifold's metric $g$ is comparable (topologically equivalent) to the Euclidean metric $\| \cdot \|_2$ on $\R^L$ which contains $M$. This yields the inclusion of Sobolev spaces:
\begin{align}
\label{eq:sobolev_comparison}
  H^1\left( [0,1] ; \left(M, g\right) \right) \ \subset \ 
& H^1\left( [0,1] ; \R^{L} \right) \ .
\end{align}
The norm of a function $\varphi \in H^1 := H^1\left( [0,1] ; \R^{L} \right)$ is:
\begin{align}
\label{def:H1_scalar_product}
\left\| \varphi \right\|_{H^1} & := \sqrt{ \|\varphi(0)\|_2^2 + \int_0^1 dt \ \|\varphi'(t)\|_2^2 } \ .
\end{align}

\subsection{Proof of Theorem \ref{thm:main}}
\label{subsection:proof}

Our proof details successively the existence of geodesics as minimizers of the energy functional, their smoothness, the fact that they have constant speed, the (generic) uniqueness and finally the local diffeomorphism property.

Although non-trivial, these arguments are classical for a mathematician seasoned in Riemannian geometry.

\subsubsection{Existence of minimizers to Eq. \texorpdfstring{\eqref{def:variationalProblem}}{}}
For this subsection, we follow the direct method in the calculus of variations \cite{DACOROGNA}. As $M$ is path connected, the minimization problem \eqref{def:variationalProblem} is over an non-empty set and there exists a minimizing sequence $\left( \gamma^{n} \ ; \ n \in \N \right)$ in $H^1\left( [0,1]; (M,g) \right)$. It satisfies:
\begin{align*}
\gamma^{n}_{t=0} & = (f_1, b_1) \ ,\\
\gamma^{n}_{t=1} & \in F_{b_2}  \ ,\\
\Ec\left( \gamma^{n} \right) & 
   \stackrel{n \rightarrow \infty}{ \longrightarrow}
   \underset{
   \substack{
      \gamma \in H^1\left( [0,1]; (M,g) \right), \\
      \gamma_{t=0} = (f_1, b_1), \\
      \gamma_{t=1} \in F_{b_2}
   }
   }{\inf} \ 
   \Ec\left( \gamma \right) \ .
\end{align*}
The following classical compactness argument allows to find a converging subsequence. First, notice that for all functions $\varphi \in H^1$ and $0 \leq t_1 \leq t_2 \leq 1$:
$$
     \| \varphi(t_1)-\varphi(t_2) \|_2
\leq \int_{t_1}^{t_2} dt \ \|\varphi'(t)\|_2 
\leq \sqrt{t_2-t_1} \ \|\varphi\|_{H^1} \ ,
$$
where we successively used the triangular inequality and the Cauchy-Schwarz inequality. Because of the hypothesis {\bf (H)} which led to the inclusion \eqref{eq:sobolev_comparison}, the sequence $\left( \gamma^n , n \in \N \right)$ is bounded in $H^1$ and the previous inequality shows that $\left( \gamma^n , n \in \N \right)$ is equicontinuous and uniformly bounded. Thanks to the Ascoli-Arzela theorem, there exists a subsequence $n_k$ such that we have the convergence to a continuous path:
$$ \gamma^*
 = \lim_{k \rightarrow \infty} \gamma^{n_k}
\ .$$
This convergence is in the uniform topology and we not know yet that $\gamma^* \in H^1$. This is obtained by lower-semicontinuity of the norm function, so that:
$$ \| \gamma^* \|_{H^1}
   \leq 
   \liminf_{k \rightarrow \infty} \| \gamma^{n_k} \|_{H^1} < \infty \ .
$$
Using the same argument for the energy functional:
$$ \Ec\left( \gamma^* \right)
 \leq \liminf_{k \rightarrow \infty} \Ec\left( \gamma^{n_k} \right)
 = \inf_{\gamma} \ \Ec\left( \gamma \right)
 \ .
$$
In the end, $\gamma^*$ is indeed an energy minimizing geodesic.

\subsubsection{Smoothness of geodesics}
Let $\left( x_k \right)_{1 \leq k \leq \dim M}$ be a local chart for the manifold $M$. Now that we have the existence of a minimizing curves in $H^1$, let us write the Euler-Lagrange equation corresponding to the fact that minimizers are critical points of the energy. 

To that endeavor, consider a minimizer $\gamma$ with fixed initial position $\gamma_{t=0} = (f_1, b_1)$ and fixed final point $\gamma_{t=1} = (f_2, b_2)$. We will prove that it necessarily satisfies the Euler-Lagrange equation for all $t \in [0,1]$:
\begin{align}
\label{eq:EulerLagrange}
0 & = \half \left( \langle \dot{\gamma}_t, \frac{\partial g}{\partial x_k}(\gamma_t) \dot{\gamma}_t \rangle \right)_{k}
        -
        \frac{d}{d t} \left( g(\gamma_t) \dot{\gamma}_t \right) \ .
\end{align}
Thanks to the Leibniz rule and rearranging the equation, one obtains the second order ODE:
\begin{align}
\label{eq:EulerLagrangeODE}
    g(\gamma_t) \ddot{\gamma}_t 
= & \half \left( \langle \dot{\gamma}_t, \frac{\partial g}{\partial x_k}(\gamma_t) \dot{\gamma}_t \rangle \right)_{k}
    -
    \frac{d}{d t}\left( g(\gamma_t) \right) \dot{\gamma}_t \ .
\end{align}

For any smooth path $\delta: [0,1] \rightarrow M$ with $\delta_0 = \delta_1 = 0$, which we use as a perturbation, the fact that $\gamma$ is a critical point yields:
\begin{align*}
0 = & \lim_{\varepsilon \rightarrow 0} \frac{\Ec(\gamma + \varepsilon \delta) - \Ec(\gamma)}{\varepsilon} \\
  = & \lim_{\varepsilon \rightarrow 0} 
      \half \varepsilon^{-1}
      \int_0^1 dt \langle \dot{\gamma}_t + \varepsilon \dot{\delta}_t, \ 
                        g( \gamma_t ) \cdot (\dot{\gamma}_t + \varepsilon \dot{\delta}_t) \rangle
      \\
    & + \half \int_0^1 dt
      \langle \dot{\gamma}_t + \varepsilon \dot{\delta}_t, \ 
              \frac{g( \gamma_t + \varepsilon \delta_t)-g( \gamma_t )}{\varepsilon} \cdot 
              (\dot{\gamma}_t + \varepsilon \dot{\delta}_t) \rangle
      - \half \int_0^1 dt \varepsilon^{-1} \langle \dot{\gamma}_t, \ 
                        g( \gamma_t ) \cdot \dot{\gamma}_t \rangle
      \\
  = & \int_0^1 dt
      \lim_{\varepsilon \rightarrow 0} 
      \half \langle \dot{\gamma}_t, \ \frac{g( \gamma_t + \varepsilon \delta_t)-g( \gamma_t )}{\varepsilon} \cdot \dot{\gamma}_t \rangle
    + \int_0^1 dt
      \langle \dot{\gamma}_t, \ g( \gamma_t ) \cdot \dot{\delta}_t \rangle
      \\
  = & \int_0^1 dt \left(
      \half \sum_k \left(\delta_t\right)_k \langle \dot{\gamma}_t,  \frac{\partial g}{\partial x_k}(\gamma_t) \dot{\gamma}_t \rangle
      + \langle \dot{\gamma}_t, g(\gamma_t) \dot{\delta}_t \rangle
      \right) \ .
\end{align*}
The above expression assumes that the local chart $\left( x_k \right)_{k}$ is valid on the entire path. However, this is easily dealt with by taking perturbations $\delta$ supported on local charts. As such, without loss of generality, we assume for the rest of the argument that the chart is global. Upon performing an integration by parts on the second integral, we obtain:
\begin{align*}
0 = & \int_0^1 dt \left(
      \half \sum_k (\delta_t)_k \langle \dot{\gamma}_t,  \frac{\partial g}{\partial x_k}(\gamma_t) \dot{\gamma}_t \rangle \right) \\
    & + \ \left[ \langle \dot{\gamma}_t, g(\gamma_t) \delta_t \rangle \right]_0^1 
      - \ 
      \int_0^1 dt \left(
      \sum_k (\delta_t)_k \ \frac{d}{d t}\left( g(\gamma_t) \dot{\gamma}_t \right)_k \right)  \ .
\end{align*}
Since $\delta_0 = \delta_1 = 0$, the boundary terms $\left[ \langle \dot{\gamma}_t, g(\gamma_t) \delta_t \rangle \right]_0^1$ vanish. For all smooth $\delta: [0,1] \rightarrow M$, we have:
\begin{align}
\label{ref:zeroIntegral}
0  = &
      \sum_k
      \int_0^1 dt \ (\delta_t)_k \left(
      \half \langle \dot{\gamma}_t,  \frac{\partial g}{\partial x_k}(\gamma_t) \dot{\gamma}_t \rangle
      - \frac{d}{d t}\left( g(\gamma_t) \dot{\gamma}_t \right)_k \right)
      \ .
\end{align}
By density of smooth functions in $L^2([0,1], M)$, we obtain:
$$
0  = \sum_k
     \int_0^1 dt \ \left|
     \half \langle \dot{\gamma}_t,  \frac{\partial g}{\partial x_k}(\gamma_t) \dot{\gamma}_t \rangle
     - 
     \frac{d}{d t}\left( g(\gamma_t) \dot{\gamma}_t \right)_k \right|^2 \ ,
$$
and the integrand must be zero. This yields indeed the Euler-Lagrange equations \eqref{eq:EulerLagrange} coordinate-wise.

Multiplying \eqref{eq:EulerLagrangeODE} by $g(\gamma_t)^{-1}$, one sees that $\ddot \gamma$ is a smooth function of $\dot \gamma$ and $\gamma$, which is indeed a second order ODE. Because of the Cauchy-Lipschitz Theorem, everything is determined from the datum $\left( \gamma_{t=0}, \dot{\gamma}_{t=0} \right)$. Also because of the ODE's coefficients are smooth, the geodesic $\gamma$ is smooth as well.

\subsubsection{ Geodesics have constant speed}

\begin{lemma}
\label{lemma:constant_speed}
If $\gamma$ is a geodesic, then necessarily:
$$ 
\forall t \in [0,1], \ 
\half \langle \dot{\gamma_t}, g(\gamma_t) \dot{\gamma_t} \rangle = \Ec(\gamma) \ .$$
As a consequence, the map
$ t \mapsto \half \int_0^t dt \ \langle \dot{\gamma_t}, g(\gamma_t) \dot{\gamma_t} \rangle $
grows linearly.
\end{lemma}
\begin{proof}
Consider the Euler-Lagrange equation \eqref{eq:EulerLagrange} and form the scalar product with $\dot{\gamma_t}$. 
$$
0 = -\half \langle \dot{\gamma_t},
         \sum_k (\dot{\gamma}_t)_k \frac{\partial g(\gamma_t)}{\partial x_k} \dot{\gamma_t} \rangle
    + \langle \dot{\gamma_t}, \frac{d}{dt}\left( g(\gamma_t) \dot{\gamma_t} \right) \rangle \ .
$$
Using the chain rule, and then twice the Leibniz rule, we obtain:
\begin{align*}
0 = & \ -\half \langle \dot{\gamma_t}, \frac{d}{dt}\left( g(\gamma_t) \right) \dot{\gamma_t} \rangle
    + \langle \dot{\gamma_t}, \frac{d}{dt}\left( g(\gamma_t) \dot{\gamma_t} \right) \rangle \\
  = & \ \half \langle \dot{\gamma_t}, \frac{d}{dt}\left( g(\gamma_t) \right) \dot{\gamma_t} \rangle     + \langle \dot{\gamma_t}, g(\gamma_t) \ddot{\gamma_t} \rangle \\
  = & \ \half \frac{d}{dt} \langle \dot{\gamma_t}, g(\gamma_t) \dot{\gamma_t} \rangle \ .
\end{align*}
As such, $\half \langle \dot{\gamma_t}, g(\gamma_t) \dot{\gamma_t} \rangle$ is a constant, which is necessarily $\Ec(\gamma)$ by virtue of the expression \eqref{eq:energy}.
\end{proof}

An interesting feature of this property is that it can be used to test the convergence of the algorithm described in Subsection \ref{subsection:geodesics_numerics}.

\subsubsection{Generic uniqueness of geodesics between fibers} 
This is the delicate part. In general, it is well-known that minimizing geodesics are not always unique. The simple example to have in mind is the sphere, where opposite points have infinitely many geodesics joining them i.e great circles. However, pairs of opposite points have zero Lebesgue measure among all possible pairs. This example is in fact archetypal, and the goal of this section is to prove that in most cases of applications, any noise would play in our favor and guarantee unique geodesics.

Now, consider a minimizing geodesic $\gamma$, solution to problem \eqref{def:variationalProblem}. Necessarily, the geodesic hits the fiber $F_{b_2}$ along a normal speed. As such, the geodesic can be seen in reversed time by writing $\widetilde{\gamma}_t = \gamma_{1-t}$. Thus, we obtain a geodesic $\widetilde{\gamma}$ which starts at $(f_2, b_2)$ with an initial speed normal to $F_{b_2}$.  Clearly, $\gamma$ is uniquely determined if and only if $\widetilde{\gamma}$ is uniquely determined among minimal geodesics starting from $F_{b_2}$ with a normal unit speed and passing through $(f_1, b_1)$.

Recall that the the cut locus of a point, written $\textrm{Cut}(\{p\})$, is the set where geodesics starting from $p \in M$ are no longer minimal - see \cite{GHL90}[$\mathsection 2.112$]. By \cite{GHL90}[Scholium $\mathsection 3.78$], non-uniqueness of geodesics starting from $p$ implies belonging to $\textrm{Cut}(\{p\})$. In fact, as we shall see in the next paragraph, the notion of cut locus generalizes to the case of any closed submanifold $N$, which we denote by $\textrm{Cut}(N)$. In our case, we take $N=F_{b_2}$ and in the end, the proof requires the two following facts:
\begin{itemize}
    \item $\textrm{Cut}(N)$ is closed with vanishing Lebesgue measure.
    \item Uniqueness of geodesics $\widetilde{\gamma}$ is provided as soon as
$$ 
(f_1, b_1) \notin \textrm{Cut}(N) \ .
$$
\end{itemize}

We start with proving the first fact as it allows to properly define $\textrm{Cut}(N)$. Let $U N$ be the unit normal bundle of $N$ and 
$$ \pi_N: U N \rightarrow N \subset M$$
be the natural projection in this context. Following \cite{IT01}, a unit speed geodesic segment $\gamma: [0,a] \rightarrow M$ emanating from $N$ is called an $N$-segment if $t = d(N, \gamma_t)$ where $d$ is the natural distance on $(M,g)$. Also for $v \in U N$, define $\rho(v) \in \R_+ \cup \{\infty\}$ as the cut time:
$$ \rho(v) := \sup \left\{ t \geq 0 \ | \ \gamma_{|[0,t]} \textrm{ is an $N$-segment} \right\} \ .$$
This generalizes the time up to which a geodesic $c_v: \R_+ \rightarrow M$ starting at a point with a unit speed $v$ remains minimal. We obtain the definition of the cut locus with respect to a submanifold $N$ as \cite{IT01}[Definition 2.3]:
$$ 
  \textrm{Cut}(N)
  :=
  \left\{ \ 
  c_v\left( \rho(v) \right) \ | \ v \in N, \ \rho(v) < \infty 
  \ \right\} \ .
$$
Then \cite{IT01}[Theorem B] yields that $\rho$ is Lipschitz regular (where finite). Thus $\textrm{Cut}(N) \subset M$ is closed and has Hausdorff dimension at most $\dim M -1$, being the graph of a regular function. As such, the cut locus has indeed zero Lebesgue measure.

For the second fact, consider two minimal geodesics $\widetilde{\gamma}_1$ and $\widetilde{\gamma}_2$ exiting $N$ with a normal unit speed and passing through $(f_1, b_1)$. Necessarily, they cross each other transversally at $(f_1, b_1)$. One concludes that $\widetilde{\gamma}_1 \neq \widetilde{\gamma}_2$ implies $(f_1, b_1) \in \textrm{Cut}(N)$ via essentially the same argument as for a point - see the elegant pictorial proof in \cite{GHL90}[Fig 2.20].

\subsubsection{The local diffeomorphism property of \texorpdfstring{\eqref{def:correspondence}}{} via a flow} 
By uniqueness of the geodesics outside of the cut locus, the map $\Cc_{b_1}^{b_2}$ is well-defined. By identifying $F_{b_1}$ and $F_{b_2}$ to the standard fiber $F$, one can see $ \Cc_{b_1}^{b_2}$ as a map from an open set of $F$ to $F$. Furthermore, because of hypothesis {\bf (H)}, the Euclidean topology of $\R^L$ is comparable to $M$'s metric and the convergence to the identity map as $b_1 \rightarrow b_2$ is obvious.

Now, let us deal with the local diffeomorphism property. Since the cut locus $\textrm{Cut}(F_{b_2})$ is closed in $M$, so is $\textrm{Cut}(F_{b_2}) \cap F_{b_1}$ and there is a neighborhood $U$ of $(f_1, b_1) \in F_{b_1}$ where the map $\Cc_{b_1}^{b_2}$ is well-defined. By uniqueness of minimal geodesics starting from $F_{b_2}$ in a normal fashion, $\left( \Cc_{b_1}^{b_2} \right)_{|U}$ is injective. In the end, only smoothness and proving an open image needs to be addressed. This is better seen via the following tubular flow argument.  

Start by fixing the minimal geodesic $\gamma^*$ such that:
$$ \gamma_{t=0}^* = (f_1,b_1), \quad \gamma_{t=1}^* = (f_2,b_2) = p \ ,$$
all of the path $\gamma^*$ does not intersect the cut locus, by definition of $\textrm{Cut}(N)$. By the tubular neighborhood theorem, there exists a neighborhood of the curve $\gamma^*$ not intersecting the cut locus. In fact, notice that this neighborhood can be taken as an open tube made of flow lines. To that endeavor, let $-v \in T_p M$ be the unit speed of $\gamma^*$ at the end point, then consider a small neighborhood $\Nc$ of $(p,v) \in U N = U F_{b_2}$ and its image via the geodesic flow. By restricting further $\Nc$ if needed, we obtain a neighborhood of $\gamma^*$ of smooth geodesic flow lines. Restricting to the flow lines which intersect $F_{b_1}$ gives a (local) diffeomorphism flow.

\subsection{On the numerics of geodesics}
\label{subsection:geodesics_numerics}

A first possible route is to numerically solve the geodesic second order ODE \eqref{eq:EulerLagrange}, for instance via a standard Euler scheme. This method requires the computation of the metric tensor exactly, and its derivative. Equivalently, because of the expression \eqref{eq:g_tensor}, it requires the computation  of the Jacobian $\nabla \Psi_\theta$ and the Hessian $\nabla^2 \Psi_\theta$ along the geodesic path $\gamma$. However, the computation of higher order derivatives is much too costly in automatic differentiation packages such as pyTorch \cite{PYTORCH} and TensorFlow \cite{TF}.

In order to leverage the power of automatic differentiation, we need a method of order $1$. As such, we directly implement a gradient descent which minimizes of the energy functional along all paths:
\begin{align}
\label{eq:argmin_energy}
& \underset{ \gamma \in H^1\left( [0,1]; (M,g) \right) }{\textrm{Argmin}} \ \Ec(\gamma) \ .
\end{align}

Recall that in our case of application $M = F \times B$ with $B=\R^n$ and $F=[-1, 1]^m$. This yields, $M \subset \R^{L}$ with $L=n+m$. In order to solve the problem numerically, we use several standard approximations. We start by discarding the restriction to paths that are $M$-valued. Since $M$ has non-empty interior inside $\R^{m+n}$ and is convex, optimal paths on $\R^{m+n}$ have little chance of exiting $M$. The optimization is thus simplified by considering all the $\gamma \in H^1\left( [0,1], \R^{n+m} \right)$. Furthermore, the optimization problem is made finite-dimensional by expanding the path $\gamma \in H^1\left( [0,1], \R^{m+n} \right)$ on a Hilbert basis and keeping only finitely many coefficients. This is exactly the spirit of the Ritz-Galerkin method in numerical analysis. Finally, the energy functional is approximated by a Riemann sum with time step $\Delta t$. In doing so, we obtain the computable minimizing problem:
\begin{align}
\label{eq:approx_argmin_energy}
& \underset{ \gamma \in V_N }{\textrm{Argmin}} \ \Ec_{\Delta t}(\gamma) \ ,
\end{align}
where $V_N \subset H^1\left( [0,1], \R^{n+m} \right)$ is a finite dimensional space and 
$$ E_{\Delta t}(\gamma)
 = \sum_{k=1}^{1/\Delta t}
   \frac{
         \left| \Psi_\theta\left( \gamma_{(k+1)\Delta t} \right)
                -
                \Psi_\theta\left( \gamma_{k \Delta t} \right)
         \right|^2
        }
        {\Delta t} \  \ .
$$

Let us conclude this section by giving a precise description of the space $V_N$. We chose to build it from the Faber-Schauder system 
$$ \left( s_0, s_1, s_{j,k} \right)_{j \in \N, \ k \in \llbracket 0, 2^j-1 \rrbracket} \ ,$$
also commonly known as the basis of hat functions - see Eq.(2.3) from \cite{Triebel}[Chapter 2]. This is given for $t \in [0,1]$ by:
$$ s_{0}(t) = 1 \ ,$$
$$ s_{1}(t) = t \ ,$$
and for all $j \in \N$ and $k \in \llbracket 0, 2^j-1 \rrbracket$:
$$ s_{j,k}(t) = 2^{1+\frac{n}{2}} \int_0^t \psi_{n,k}(u) du$$
where $\psi_{j,k}$ is the Haar system defined from the Haar wavelet $\psi = \mathds{1}_{[0,\half]}-\mathds{1}_{[\half, 1]}$ as
$$ \psi_{j,k}(t) = 2^{\frac{n}{2}} \psi\left( 2^n t - k \right) \ .$$
The Haar system along with the constant function $1$ forms an orthonormal basis of $L^2\left( [0,1], \R \right)$. Therefore, given the norm definition in Eq.  \eqref{def:H1_scalar_product}, the Faber-Schauder system is an orthonormal basis of $H^1\left( [0,1], \R \right)$. Upon tensoring, we readily obtain a basis of $H^1\left( [0,1], \R^{n+m} \right)$ indexed by triple indices $(i,j,k)$:
$$ \forall t \in [0,1], \ 
   s^i_{j,k}(t) := \sum_{i=1}^{\dim M} s_{j,k}(t) \ e_i \ ,
$$
where $\left( e_i \right)_{1 \leq i \leq \dim M}$ is the canonical basis of $\R^{\dim M}$. In the end, we worked with
\begin{align*}
V_N := & \Span_\R\left\{ \ s^i_{j,k} \ | \
                         i \in \llbracket 1, \dim M \rrbracket, \right.\\
       & \quad \quad \quad \quad \quad \quad \left.
                         j \in \llbracket 0, N-1 \rrbracket, \ 
                         0 \leq k < 2^j \right\} \ ,
\end{align*}
and we have:
$$ \dim V_N = (2^N + 1) \dim M \ .$$

\section{Fibered Auto-Encoders}
\label{section:network}

A Fibered Auto-encoder (FAE) is essentially a standard auto-encoder whose latent space is stratified into base space $B$ and a fiber space $F$. However, to enforce the supervised geometric disentanglement between $B$ and $F$ and to improve the quality of generated samples, we have added auxiliary objectives to the reconstruction loss: (i) to enforce the geometric disentanglement between base and fiber spaces, (ii) to improve the quality of the reconstructions.

Throughout this work, we used only fully-connected layers. The architecture we propose can however accommodate any other types of layers. Throughout this section, we discuss Fig. \ref{fig:nnArchitecture} which serves as a comprehensive drawing.

\input{figures/figNeuralNetworkStructure.tex}

\subsection{Auto-encoder definition}
\label{subsection:AE_def}

Just like in a standard auto-encoder \cite{AE1, AE2}, the encoder part of an FAE receives the sample $X$ and terminates in a bottleneck layer with a dimensionality lower than that of the input. The output $f$ of the bottleneck layer is concatenated with the output $b$ of an embedding layer that receives the condition label $c$. For a pair of sample and condition $(X, c)$, the output of the FAE is computed as 
$\widehat{X} := \Psi_{\theta_d}\left( f, b \right) $ with 
$f := \Phi_{1, \theta_e} \left( X \right)$ and 
$b := \Phi_{2, \theta_m} \left( c \right)$ .
In fact, $\Psi_{\theta_d}$ is the decoder function used in the previous section.

For convenience, we package all the parameters optimized for reconstruction as $\theta = (\theta_{e}, \theta_{m}, \theta_{d})$. We shall write:
\begin{align}
\label{def:varphi}
 	(f, b) := & \ \Phi_{\theta}\left( X, c \right)
 	        \ = \ \left( \Phi_{1, \theta_e} \left( X \right) \ ,
	                 \Phi_{2, \theta_m} \left( c \right) \right) \ .
\end{align}
The variable $f \in F$ represents the coordinates of the sample in the fiber space, while $b \in B$ represents the coordinates of the condition in the base space. 

In accordance with the metric \eqref{def:targ_space} of $\Xc$, the reconstruction loss we minimize is the Mean Squared Error (MSE):
\begin{align}
\label{def:MSE}
     \Lc^{mse}\left( \theta \ ; \Dc \right) 
:= & \frac{1}{N} \sum_{i=1}^N \left\| X_i - \widehat{X}_i \right\|_2^2 \ ,
\end{align}
where the dataset is $\Dc := \left( (X_i, c_i) \right)_{1 \leq i \leq N}$.

To help with training we added skip connections that propagate the latent variables $(f,b)$ to every layer of the decoder. The output $H_i$ of a decoder's $i$-th layer is defined by:
$$ H_{i} := \Psi_{i, \theta_d}\left( f, b, H_{i-1} \right) \ .$$

Once trained, samples can be generated from the network by sampling from $F$ and $B$ and using the decoder.

Finally, the auto-encoder uses the \emph{sine} function as non-linearity, that has recently been shown to improve detail representation of complex natural signals in deep neural networks \cite{SITZMANN}. This choice is also motivated by two reasons: (i) a smooth decoder $\Psi_\theta$ is needed in order to induce a structure of smooth Riemannian manifold (See Eq. \eqref{def:gTensor}). (ii) in contrast to the \emph{sigmoid} and \emph{tanh} functions, the \emph{sine} function has the added advantage of not saturating. Since \emph{sine} outputs are in $[-1,1]$, it follows that $F = [-1,1]^{n}$, where $n$ is the number of neurons in the bottleneck. Training algorithm is given in Algorithm \ref{algo:recons_update}.

\begin{algorithm}[H]
   \caption{FAE reconstruction training algorithm for one batch}
\begin{algorithmic}
   \STATE {\bfseries Name:} reconstruction-update
   
   \STATE {\bfseries Input:}
   
    Training data:
    $(X, \widehat{X})$,
    
    Parameters:
    $ \theta$
    
    Learning rates:
    $\mu$
 
   \STATE $g_{\theta} \leftarrow 
   \nabla_{\theta} \frac{1}{n} \sum_{i=1}^n \left| X_i - \widehat{X}_i \right|^2$
   
   \STATE $\theta \leftarrow \theta - \textrm{Adam}(\mu, g_{\theta})$
\end{algorithmic}
\label{algo:recons_update}
\end{algorithm}

\subsection{Geometrically disentangling fiber and base spaces}

Because the encoder $\Phi_{1, \theta_e} $is exposed to the training samples, information about sample condition could become jointly encoded into $F$ and $B$.
Fiber and base space would then be entangled, which could hinder both the generative process and the transport. Here we found that the principles of \emph{domain adversarial training} \cite{DANN1, DANN2} can be reliably applied to our case. The simple addition of a classifier predicting the condition $c$ from $F$ coupled with a Gradient Reversal Layer (GRL) worked well in our case (Fig. \ref{fig:MNIST_no_ap}). This method is identical to the one introduced by \cite{DANN1}, however we refer to it in the context of this work as \emph{condition adversarial training}, to highlight the fact that it is used to remove condition information from $F$. 
Other methods explicitly minimizing the mutual information \cite{MINE} could be explored in the future.

\subsubsection{Domain adversarial training}

The original goal of domain adversarial training \cite{DANN1, DANN2} is to make a classifier insensitive to changes in the domain of input. Here changing domains manifests in shifts in data distributions that are not related to the nature of the classification task. For a digit classifier, hand-written numbers and house numbers are two possible domains. In that case, being domain insensitive means that the classifier uses discriminating features that are domain independent. 

For an input couple data-label $(X, y)$, say a picture of a digit and its label, Ganin et al. define a latent variable $f$ and an estimated label $\widehat{y}$ as follows:
\begin{align}
	f           & := E_{\theta_e}\left( X \right) \ , \\
    \widehat{y} & := LC_{\theta_c}\left( f \right)
                   = LC_{\theta_c} \circ E_{\theta_e}\left( X \right) \ ,
\end{align}
where $E_{\theta_e}$ is an encoder, and $LC_{\theta_c}$ a label classifier.

The presence of domain information in the latent space is estimated by adding a domain classifier $DC_{\theta_d}$ that takes a latent variable $f$ as input and predicts the domain $\widehat{d}$ from it. 

\begin{align}
	\widehat{d} & := DC_{\theta_d} \left( f \right) = DC_{\theta_d} \circ E_{\theta_e}\left( X \right)\ .
\end{align}

In this context, there are two loss functions: $\Lc_{LC}\left( \theta_{e}, \theta_{c} \right)$ for the label classifier and $\Lc_{DC}\left( \theta_{e}, \theta_{d} \right)$ for the domain classifier. As explained in \cite{DANN1}[Fig 1], removing domain information from the latent variable is reduced to minimizing $\Lc_{LC}$ with respect to $(\theta_e, \theta_e)$, minimizing $\Lc_{DC}$ with respect to $\theta_d$ while maximizing it with respect to $\theta_e$.

By defining an aggregate loss on the sample $\left( (X_i, y_i) \right)_{1 \leq i \leq n}$ that is:
\begin{align}
\label{def:dann_loss}
    & \Lc^{adv}\left( \theta_{e}, \theta_{c}, \theta_{d} \ ; \ \left( (X_i, y_i) \right)_{1 \leq i \leq n} \right)\\ 
    := & \
    \Lc_{LC}\left( \theta_{e}, \theta_{c} \right)
    - \lambda \Lc_{DC}\left( \theta_{e}, \theta_{d} \right) 
    \nonumber \ ,
\end{align}
we need to implement the saddle-point optimization program that computes $(\widehat{\theta_{e}}, \widehat{\theta_{c}}, \widehat{\theta_{d}})$:
\begin{align}
\label{def:dann_argmin}
\left( \widehat{\theta_{e}}, \widehat{\theta_{c}} \right) := & \ 
\underset{ \theta_{e}, \theta_{c} }{\textrm{Argmin}} \ 
\Lc^{adv}\left( \theta_{e}, \theta_{c}, \widehat{\theta_{d}} \right) \\
\widehat{\theta_{d}} := & \ 
\underset{ \theta_{d} }{\textrm{Argmax}} \ 
\Lc^{adv}\left( \widehat{\theta_{e}}, \widehat{\theta_{c}}, \theta_{d} \right) 
\end{align}
That is found using the sequential stochastic updates given in \cite{DANN1}[Eq. 4, 5, 6]:
\begin{align}
\label{eq:dann_update_1}
\theta_e \leftarrow \ & \theta_e - \mu\left( \frac{\partial \Lc_{LC}}{\partial \theta_e} - \lambda \frac{\partial \Lc_{DC}}{\partial \theta_e} \right) \ ,\\
\label{eq:dann_update_2}
\theta_c \leftarrow \ & \theta_c - \mu \frac{\partial \Lc_{LC}}{\partial \theta_c} \ ,\\
\label{eq:dann_update_3}
\theta_d \leftarrow \ & \theta_d - \mu \frac{\partial \Lc_{DC}}{\partial \theta_d} \ ,
\end{align}
where $\mu>0$ is a given learning rate.

A great contribution of Ganin et al. is to recast these non-standard updates as a classical gradient descent, which can be implemented in most deep learning frameworks. This is achieved using a Gradient Reversal Layer (GRL) at the level of $f$, that replaces $\frac{\partial \Lc_d}{\partial \theta_{e}}$ with its opposite. A GRL $R_\lambda$ seamlessly outputs the identity map during the forward pass, while reversing (and scaling) the gradient during the back-propagation pass. The authors carefully refer to $R_\lambda$ as a "pseudo-function" and define it as:
\begin{align}
\label{def:GRL}
    \left\{
    \begin{array}{ccc}
    R_\lambda (x)           & := & x \ ,\\
    \nabla R_\lambda(x)     & := & -\lambda \ \textrm{Id} \ .
    \end{array}
    \right.
\end{align}
where $\textrm{Id}$ is the identity matrix. The notation is clearly abusive as the two expressions are incompatible for a mathematical function, while perfectly authorized for a neural network layer.

\subsubsection{Condition adversarial training}

The principles of \emph{condition adversarial training} are identical to domain adversarial training. The only difference being that, instead of removing domain specific information from the latent space, we seek to remove condition specific information from the fiber space. Thus we replace the domain classifier $DC_{\theta_d}$ over the latent space, by a condition classifier $\overline{\Upsilon}_{\theta_{ac}}$ over $F$, and use a GRL at the $F$ level. The effects of condition adversarial training are illustrated in Fig. \ref{fig:MNIST_no_ap}.

\begin{figure}[H]
\begin{center}
\includegraphics[width=0.9\textwidth]{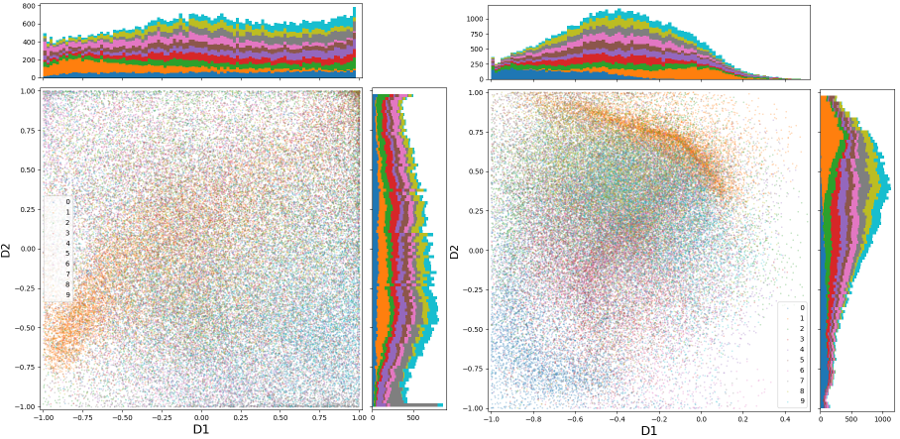}
\caption{MNIST fiber space for all conditions. Every digit is a condition represented by a color. Fiber space is $[-1,1]^2$. Left: network trained with condition adversarial training, right: without. Empirical distributions on fibers are closer to uniform, making them less distinguishable. 
}
\label{fig:MNIST_no_ap}
\end{center}
\end{figure}

The classifier $\overline{\Upsilon}_{\theta_{ac}}$ has output 
$
\ell := \overline{\Upsilon}_{\theta_{ac}}(f) \ ,
$
where $\ell$ is the vector of log-likelihoods
$$
\ell := \left( \log \P\left( \widehat{c} = c_i \right) \right)_{1 \leq i \leq n}
\ .
$$
$\overline{\Upsilon}_{\theta_{ac}}$ is trained by minimizing the cross-entropy loss with log-likehoods $\ell$, when the real condition is $c$:
\begin{align}
\label{def:xent}
\Lc^{xent}\left( \ell, c \right)
:=
- \log\left( \frac{\exp\left( \ell_{i(c)} \right)}
             {\sum_{i=1}^n \exp\left( \ell_i \right)} \right) \ ,
\end{align}
where $i(c)$ is the index of the condition $c$. $\overline{\Upsilon}_{\theta_{ac}}$ is coupled with a GRL at the level of $f$. Updates are thus defined as:
\begin{align}
\theta_e \leftarrow \ & \theta_e + \mu_{ac} \lambda \frac{\partial \Lc^{xent}(\overline{\Upsilon}(f), c)}{\partial \theta_e} \ ,\\
\theta_{ac} \leftarrow \ & \theta_{ac} - \mu_{ac} \lambda \frac{\partial \Lc^{xent}(\overline{\Upsilon}(f), c)}{\partial \theta_{ac}} \ ,
\end{align}
The loss \eqref{def:xent} is thus maximized over $\theta_{ac}$ and minimized over $\theta_{e}$. Training algorithm is given in Algorithm \ref{algo:cond_adv_update}.

\begin{algorithm}[tbh]
   \caption{FAE condition adversarial training algorithm for one batch}
\begin{algorithmic}
   \STATE {\bfseries Name:} cond-adv-update\\
   \STATE {\bfseries Input:}\\
    Training data:
    $f, c$\\
    Parameters:
    $ \theta_{e}, \theta_{ac}$\\
    Learning rates:
    $\mu_{ac1}, \mu_{ac2}$\\
    
    \STATE $\ell \leftarrow \overline{\Upsilon}_{\theta_{ac}}(f)$\\
    
    \STATE $g_{ac} \leftarrow \nabla_{\theta_{ac}} \left( - \log\left( \frac{\exp\left( \ell_{c} \right)} {\sum_{i=1}^n \exp\left( \ell_i \right)} \right) \right)$ \\ 
   \STATE $\theta_{ac} \leftarrow \theta_{ac} - \textrm{Adam}(\mu_{ac1}, g_{ac})$
 
    \STATE $g_e \leftarrow \nabla_{\theta_{e}} \left( - \log\left( \frac{\exp\left( \ell_{c} \right)}{\sum_{i=1}^n \exp\left( \ell_i \right)} \right) \right)$ \\ 
   \STATE $\theta_{e} \leftarrow \theta_{e} + \textrm{Adam}(\mu_{ac2}, g_e)$

\end{algorithmic}
   \label{algo:cond_adv_update}
\end{algorithm}

\subsection{Accessory objectives}
The MSE loss \eqref{def:MSE} on its own is not necessarily the optimal choice for every application and does not ensure reconstruction realism. Having accessory objectives aimed at improving the realism of generated samples is especially desirable for applications to datasets where the quality of samples cannot be evaluated by the naked eye. To address this issue we have included accessory objectives, that are not essential to the FAE architecture, but that the experimenter can elect to use depending on the task at hand. 

\subsubsection{Ensuring reconstruction realism}
Generative adversarial networks (GANs) \cite{GAN} are the state of the art when it comes to generating realistic images \cite{SGAN}, and have been shown to improve the quality of reconstruction of auto-encoders \cite{AAN}. Here we used a GAN objective to ensure that the samples generated by the FAE are realistic.

GANs involve the joint training of a discriminator network $D_{\theta_d}$ that learns to discriminate between generated and real samples, and a generator network $G_{\theta_g}$ that learns to counterfeit more realistic samples. The standard GAN optimization is classically defined as a min-max problem \cite{GAN}:
$$
   \min_{\theta_g} \max_{\theta_d} \ 
   \E\left[ \log\left( 1 - D_{\theta_d} \circ G_{\theta_g}(Z) \right)
          + \log D_{\theta_d}\left( X \right)
   \right]
$$
where the expectation is over the real samples $X$, and $Z$ an input used to generate counterfeited samples. The type of GAN objective can be chosen with respect to the task. To implement the adversarial loss, we add a discriminator $\Delta$ to the architecture and treat the decoder $\Psi$ as the generator. Here, we have used a standard GAN approach for images and Wassertein-GAN (WGAN) \cite{WGAN} for single-cell applications. The use of other GAN types such as MMD GANs \cite{MMD_GAN} could be explored in future work. Training algorithm is given in Algorithm \ref{algo:gan_update}.

\begin{algorithm}[tbh]
   \caption{FAE GAN training algorithm for one batch}
\begin{algorithmic}
   \STATE {\bfseries Name:} gan-update\\
   \STATE {\bfseries Input:}\\
    training data:
    $(X, \widehat{X}, c)$,\\
    parameters:
    $ \theta_{e}, \theta_{m}, \theta_{d}, \theta_{\Delta}$\\
    learning rates:
    $\mu_{\Delta_1}, \mu_{\Delta_2}$\\
 
   \STATE $g_{\theta_{gd}} \leftarrow 
   \nabla_{\theta_{\Delta}} \frac{1}{n} \sum_{i=1}^n \left[log D(x_i) + log(1 - D(\widehat{x_i})) \right]$
   \STATE $\theta_{\Delta} \leftarrow \theta + \textrm{Adam}(\mu_{\Delta_1}, g_{\theta_{\Delta}})$

    \STATE $\theta \leftarrow (\theta_{e}, \theta_{m}, \theta_{d}) $
    \STATE $g_{\theta} \leftarrow  \nabla_{\theta} \frac{1}{n} \sum_{i=1}^n log(1 - D(\widehat{x_i}))$
   \STATE $\theta \leftarrow \theta - \textrm{Adam}(\mu_{\Delta_2}, g_{\theta})$
\end{algorithmic}
   \label{algo:gan_update}
\end{algorithm}

\subsubsection{Ensuring condition discriminative features}
The GAN objective ensures the realism of generated samples, but does not prevent cases of mode collapses where reconstructions can be decoupled from their conditions, and still be deemed realistic by the discriminator. We prevent this by ensuring that the discriminative features of every condition are still present in the reconstructions.

We train a classifier ${\Upsilon}_{\theta_{c}}$ of the condition over real samples, which computes $ \ell := {\Upsilon}_{\theta_{c}}(X)$. It is trained by minimizing the cross-entropy loss with the log-likehoods $\ell$ over the real samples $X$. We then update the parameters $\theta_m$ and $\theta_d$ of respectively, $\Phi_2$ and $\Psi$, to minimize the cross-entropy of ${\Upsilon}_{\theta_{c}}(\widehat{X})$. Training algorithm is given in Algorithm \ref{algo:cond_fitting_update}.

\begin{algorithm}[tbh]
   \caption{FAE condition fitting algorithm for one batch}
\begin{algorithmic}
   \STATE {\bfseries Name:} cond-fitting-update\\
   \STATE {\bfseries Input:}\\
    Training data:
    $X, \widehat{X}, c$\\
    Parameters:
    $ \theta_{d}, \theta_{m}, \theta_{c}$\\
    Learning rates:
    $\mu_{c1}, \mu_{c2}$\\
    
    \STATE $\ell \leftarrow \Upsilon_{\theta_{c}}(X)$\\
    
    \STATE $g_{\theta_{c}} \leftarrow \nabla_{\theta_{c}} \left( - \log\left( \frac{\exp\left( \ell^{c}_{c} \right)} {\sum_{i=1}^n \exp\left( \ell_i \right)} \right) \right)$ \\ 
   \STATE $\theta_{c} \leftarrow \theta_{c} - \textrm{Adam}(\mu_{c1}, g_{\theta_{c}})$
 
    \STATE $\widehat{\ell} \leftarrow \Upsilon_{\theta_{c}}(\widehat{X})$
    
    \STATE $g_{d, m} \leftarrow \nabla_{\theta_{d}, \theta_{m}} \left( \log\left( \frac{\exp\left( \widehat{\ell_c} \right)}{\sum_{i=1}^n \exp\left( \widehat{\ell_i} \right)} \right) \right)$ \\ 
   \STATE $\left( \theta_{d}, \theta_{m} \right) \leftarrow \left( \theta_{d}, \theta_{m} \right)  - \textrm{Adam}(\mu_{c2}, g_{d, m})$
\end{algorithmic}
   \label{algo:cond_fitting_update}
\end{algorithm}

\subsection{Tying it together: training algorithm}
This section shows the details of training an FAE with the algorithm of each sub-routine presented in details. We denote by $\textrm{Adam}(\mu, g)$ a step obtained by the Adam gradient descent algorithm, with learning rate $\mu$ and gradient $g$. 
The training scheme we propose as Algorithm \ref{algo:fae_training} trains the different parts sequentially each with a specific learning rate. By setting different learning rates, one can fine tune the relative importance of the different objectives. Because the main objective is to minimize the reconstruction loss in \eqref{def:MSE}, we typically give it the highest learning rate. We used the Adam optimizer \cite{ADAM} for all objectives. Fig. \ref{fig:MNIST_curves} shows convergence for all objectives (main and accessory) on the MNIST dataset.

\begin{figure}
\begin{center}
\includegraphics[width=0.4\textwidth]{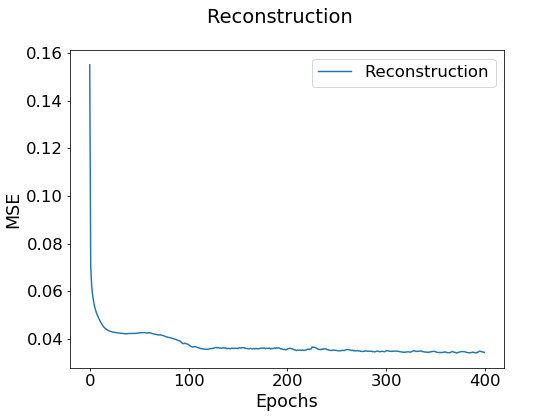}
\includegraphics[width=0.4\textwidth]{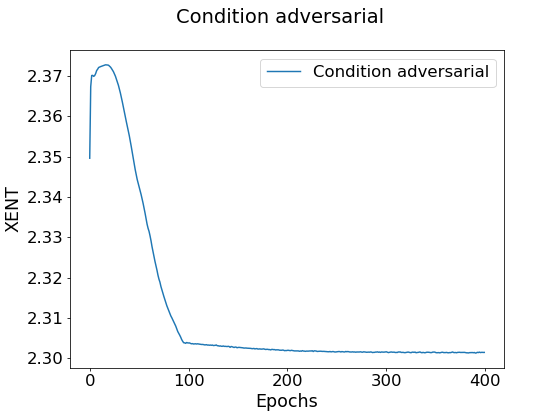}
\includegraphics[width=0.4\textwidth]{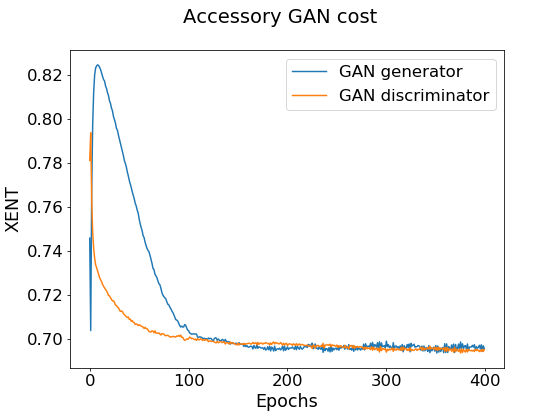}
\includegraphics[width=0.4\textwidth]{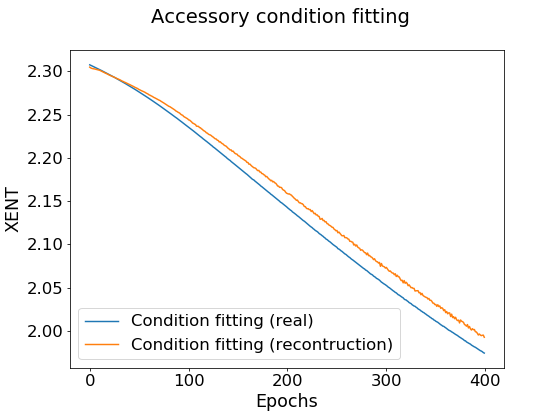}
\caption{Training curves obtained on the MNIST dataset showing convergence for the two main objectives (Reconstruction and Condition adversarial), as well as for the two accessory objectives (GAN and Condition fitting).}
\label{fig:MNIST_curves}
\end{center}
\end{figure}

\begin{algorithm}[H]
   \caption{FAE Training algorithm}
\begin{algorithmic}
   \STATE {\bfseries Input:}\\
    dataset: $\Dc$
    
    parameters: $ \theta_{e}, \theta_{m}, \theta_{d}, \theta_{ac}, \theta_{c}, \theta_{\Delta}$
    
    learning rates: $\mu_{mse}, \mu_{ac1}, \mu_{ac2}, \mu_{c1}, \mu_{c2}, \mu_{\Delta2}, \mu_{\Delta2}$
    
   \FORALL{$(X, c) \in \Dc$}
       \STATE $\widehat{X} \leftarrow \Psi_{\theta_{d}}\left( \Phi_{\theta_{e}, \theta_{m} }(X, c) \right)$
       \STATE reconstruction-update($X$, $\widehat{X}$, $\theta_{e}, \theta_{m}, \theta_{d}$, $\mu_{mse}$)
       
       \STATE $f \leftarrow \Phi_{1 \theta_{e}}\left( X\right)$
        \STATE condition-adversarial-update(
            $f, c$,
            $\theta_{e}, \theta_{ac}, \mu_{ac1}, \mu_{ac2}$
        )
        
       \STATE $\widehat{X} \leftarrow \Psi_{\theta_{d}}\left( \Phi_{\theta_{e}, \theta_{m} }(X, c) \right)$
        \STATE condition-fitting-update(
            $X$, $\widehat{X}$,
            $\theta_{d}, \theta_{m}, \theta_{c}, \mu_{c1}, \mu_{c2}$
        )
       \STATE $\widehat{X} \leftarrow \Psi_{\theta_{d}}\left( \Phi_{\theta_{e}, \theta_{m} }(X, c) \right)$
        \STATE gan-update($X, \widehat{X}$, $\theta_{e}, \theta_{m}, \theta_{d}, \theta_{\Delta}$, $\mu_{\Delta1}, \mu_{\Delta2}$)
   \ENDFOR
\end{algorithmic}
   \label{algo:fae_training}
\end{algorithm}

\section{Applications}
\label{section:Applications}
We first illustrate our method by experimenting on image datasets. We then benchmark the quality of created correspondences on single-cell datasets RNA-sequencing. Throughout this section, we contrast naive transport between two fibers, with geodesic transport.

\subsection{The principle of geodesic transport with MNIST}

\begin{figure}
\begin{center}
\includegraphics[width=0.4\textwidth]{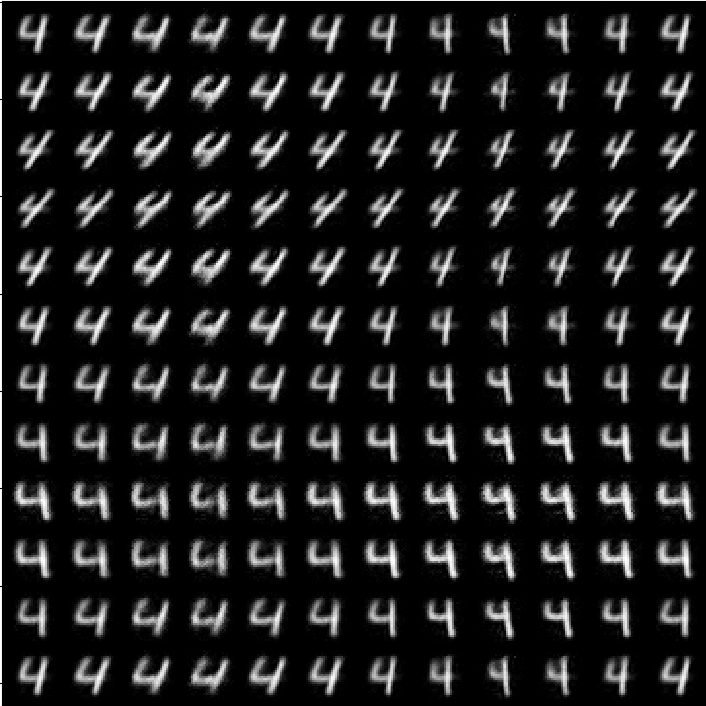}
\includegraphics[width=0.4\textwidth]{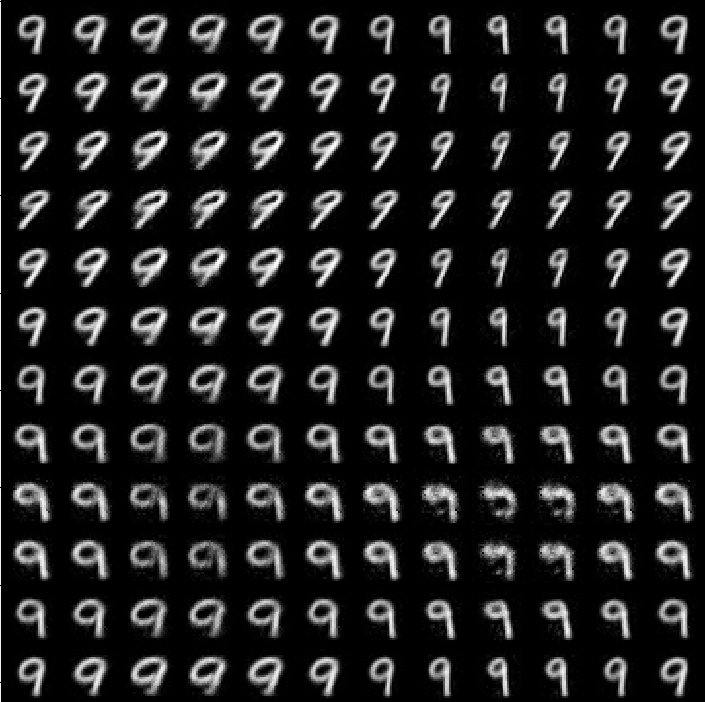}
\caption{Manifold plots for the fibers $F_4$ and $F_9$ in MNIST. Images were generated using an evenly spaced grid in the standard fiber space $F=[-1,1]^2$.}
\label{fig:MNIST_manifold_plots}
\end{center}
\end{figure}
We first experiment on MNIST. Here, the possible conditions are the digits $\left\{0, 1, \dots, 9 \right\}$, the base is $B = \R^2$ and the standard fiber is $F = [-1,1]^2$. Fig. \ref{fig:MNIST_manifold_plots} shows manifold plots obtained from an evenly spaced grid on $F_4$ and $F_9$. We can see that the reconstructions are of high quality and show a high diversity of samples despite the small bottleneck size ($2$ units). We also see that the learned latent space is contiguous as any coordinate $f \in [-1,1]^2$ yields a realistic digit. Finally, Fig. \ref{fig:MNIST_manifold_plots} shows that the learned space has an intrinsic organisation, as gradually moving on fibers gradually changes digit features. Manifolds for $F_4$ and $F_9$ exhibit similar structures, with points at the same coordinate having similar inclination and boldness. This shows that naive transport is capable of creating rather accurate correspondences between fibers of similar conditions. We do not give the manifold plot after geodesic transport as the difference with naive transport are barely perceivable.

\begin{figure}
\begin{center}
\includegraphics[trim=100 50 100 50, clip,
width=0.85\textwidth]{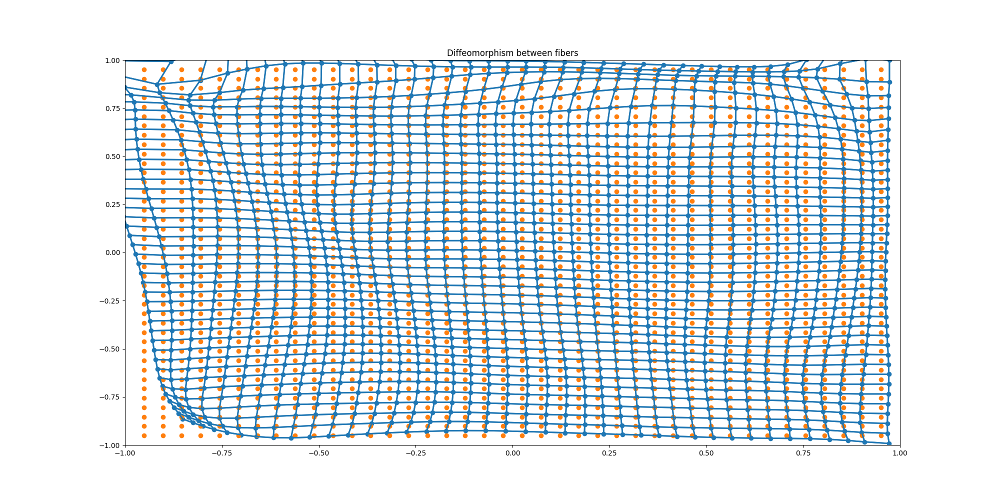}
\caption{Diffeomorphism between $F_4$ and $F_9$. The orange dots represent the original coordinates in $F_4$, the blue dots are their corresponding images in $F_9$ computed through geodesic transport.}
\label{fig:MNIST_diffeo}
\end{center}
\end{figure}

Let us now discuss the results of geodesic transport between $F_4$ and $F_9$ on the one hand, and between $F_1$ and $F_0$ on the other hand. 

\begin{figure}[tbp]
\begin{center}
\includegraphics[trim=0 0 0 0, clip, width=\textwidth]{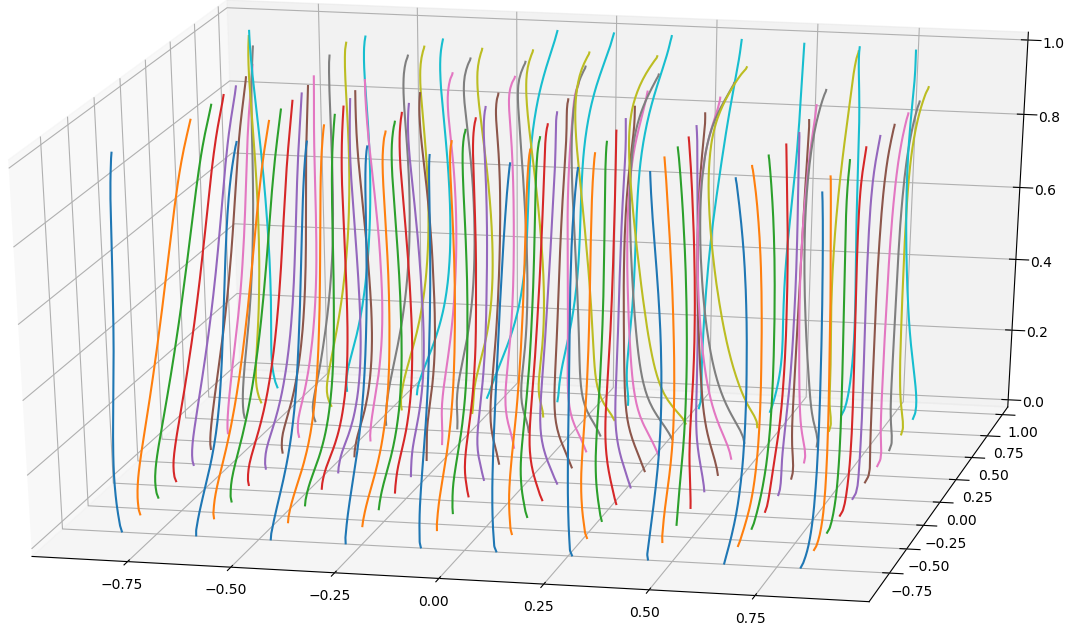}\\
\includegraphics[trim=0 0 0 0, clip, width=\textwidth]{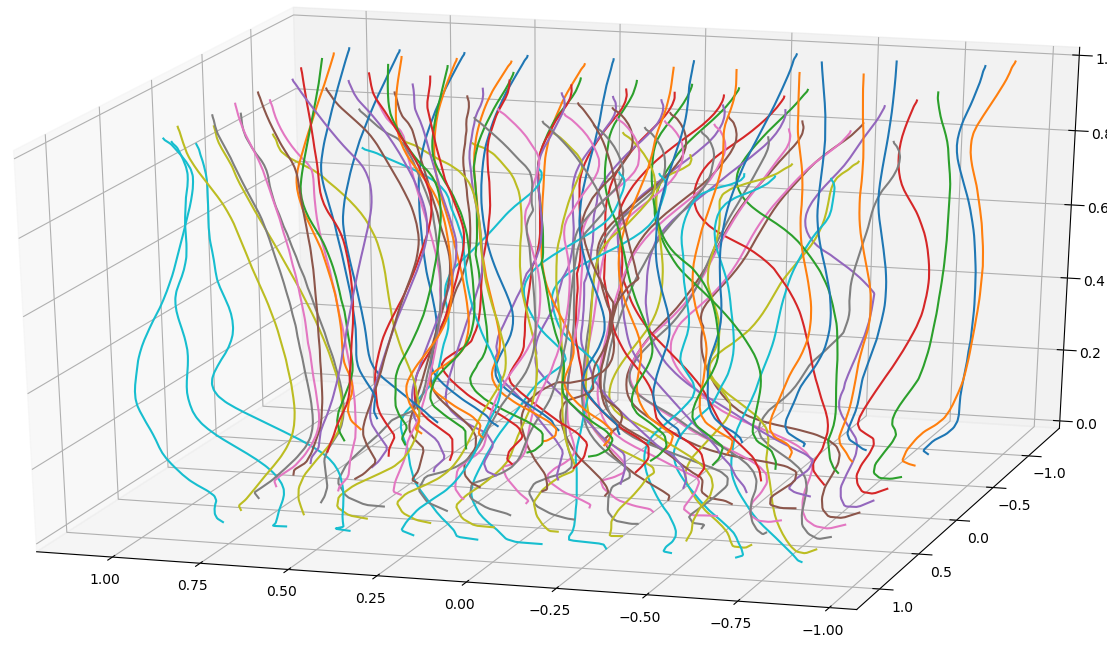}
\end{center}
\caption{Visualization of geodesic curves between two fibers isomorphic to $[-1,1]^2$. Top: From $F_4$ to $F_9$. Bottom: From $F_1$ to $F_0$. With $z$ being the height coordinate, the $z=0$ plane represents the starting fiber, while the $z=1$ plane represents the destination fiber.}
\label{fig:mnist_geodesics3d}
\end{figure}

\begin{figure}[tbp]
\begin{center}
\includegraphics[trim=200 100 150 100, clip, width=0.85\textwidth]{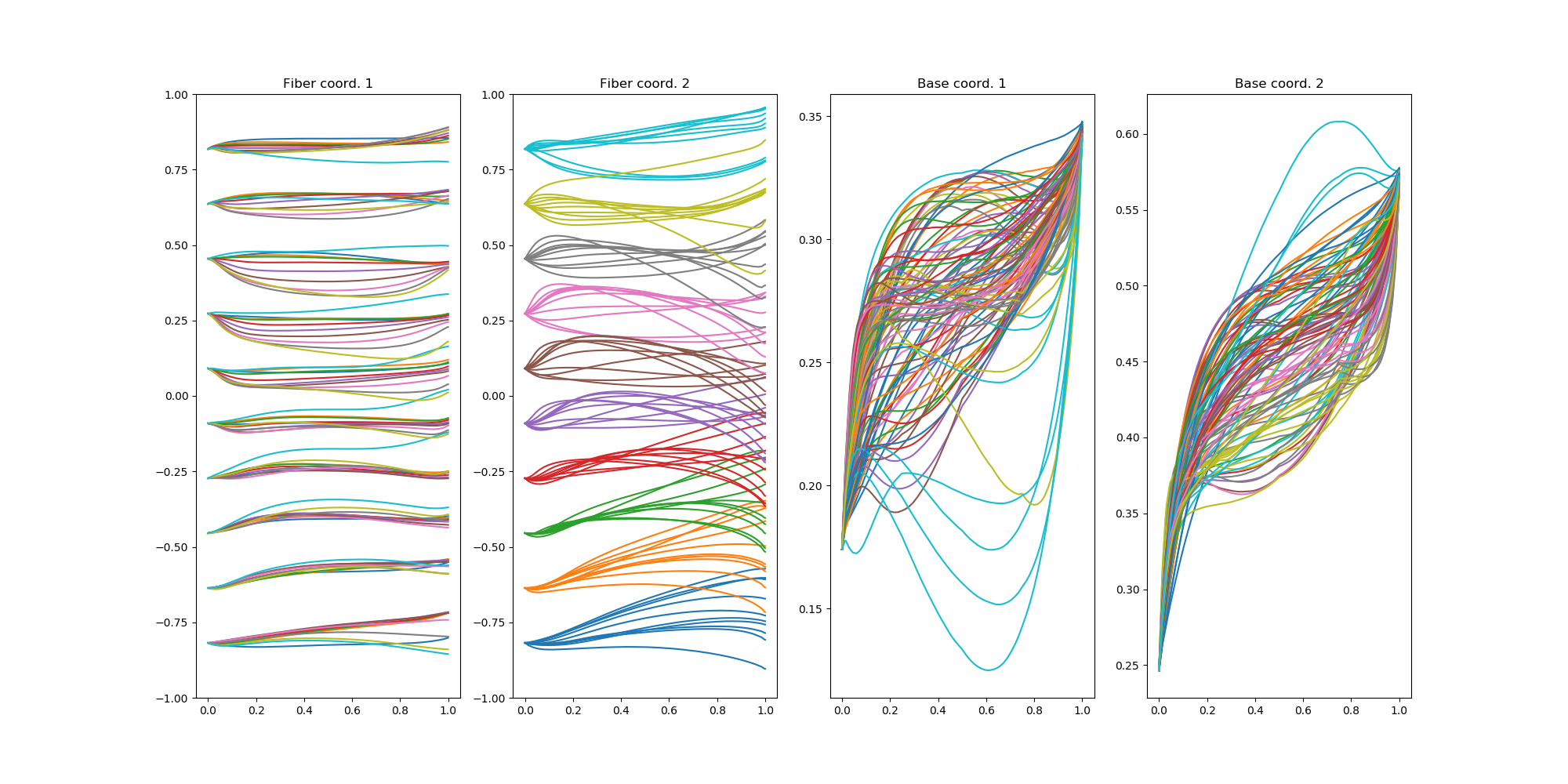}
\caption{Geodesics from $F_4$ to $F_9$. The $x$ axis displays time $t \in [0,1]$. The two left panels display coordinates in fiber space $F = [-1,1]^2$ and the two right panels display coordinates on $B=\R^2$.}
\label{fig:MNIST_geodesics}
\end{center}
\end{figure}

Fig. \ref{fig:MNIST_diffeo} gives an estimation of the diffeomorphism between fibers $F_4$ and $F_9$ as formulated in Theorem \ref{thm:main}. The starting points in $F_4$ are the orange dots obtained using an evenly spaced grid. The blue dots are the endpoints in $F_9$ for the calculated geodesics. The full 3D geodesics are given in the top portion of Fig. \ref{fig:mnist_geodesics3d}. The diffeomorphism induced through geodesic transport is globally close to the identity. This shows that naive transport can give a good approximation. However, the correction applied by geodesic transport is more apparent as we get closer to the edges. 

Fig. \ref{fig:MNIST_geodesics} displays a subset of the geodesics between fibers $F_4$ and $F_9$ used to generate Fig. \ref{fig:MNIST_diffeo}. These are not straight paths at constant speed, which shows the relevance of the Riemannian point of view. As naive transport is a good approximation in this case, geodesics on $F$ show relatively small variations. Most of the variations happen on $B$. This shows that, in the case of MNIST, a change in condition is more fundamental than a change in condition-specific features. Although these adjustments on $B$ do not change the endpoints, they still have a major impact on the interpolation between fibers.

\begin{figure}
\begin{center}
\includegraphics[trim=100 50 100 50, clip, width=\textwidth]{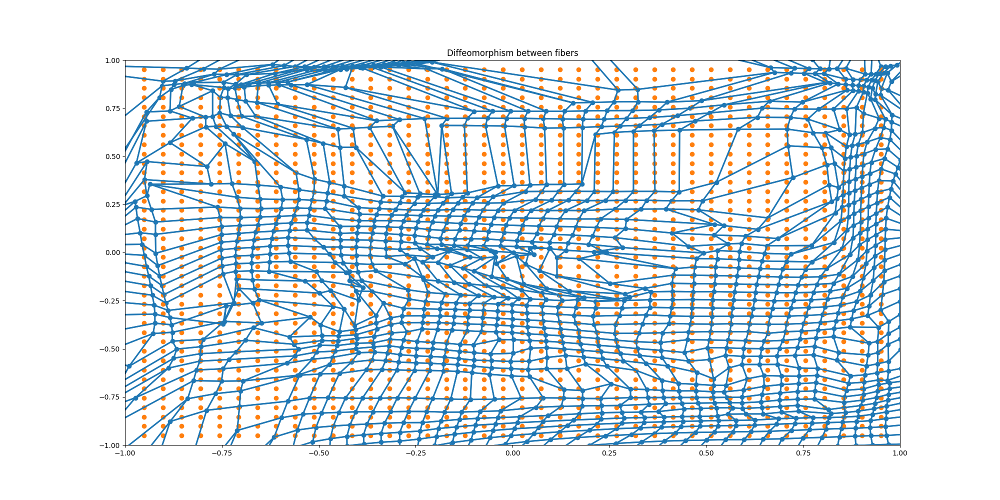}\\
\end{center}
\caption{Visualization of the (local) diffeomorphisms between two fibers, from $F_1$ to $F_0$. Orange dots represent the original coordinates in $F_1$, blue dots represent the correspondences computed through geodesic transport.}
\label{fig:mnist_diffeo_10}
\end{figure}

When examining the analogous figures for the geodesic transport from $F_0$ to $F_1$, a radically different pattern emerges.  On Fig. \ref{fig:mnist_geodesics3d}, geodesics between $F_0$ and $F_1$ (Top) are less straight than between $F_4$ and $F_9$ (Bottom). This suggests that the naive correspondence is not as strong, in accordance with intuition - the digits $4$ and $9$ are more similar. Thus, $F_1$ and $F_0$ should not be naively identified. Also, when comparing Fig. \ref{fig:MNIST_diffeo} with Fig. \ref{fig:mnist_diffeo_10}, the correspondence between $F_4$ and $F_9$ is much more global. Indeed, the correspondence between $F_1$ and $F_0$ is only locally stable, and shows more "shearing" and "tearing". At a theoretical level, this was already manifest in Theorem \ref{thm:main} where the diffeomorphism property is only local. Between $F_0$ and $F_1$, this property simply holds on smaller regions. In fact, Fig. \ref{fig:mnist_diffeo_10} is already stabilized thanks to a regularization parameter $\lambda_{\textrm{reg}}=0.02$ (See discussion in conclusion).

In the light of these experiments on MNIST, let us conclude with two intuitive criteria for assessing the quality of naive transport:
\begin{enumerate}
    \item Geodesic transport gives straight geodesic curves. In this case, naive transport transport and geodesic transport give similar results.
    \item The local diffeomorphism \eqref{def:correspondence} is stable i.e. no sharing and tearing. Upon estimating the norm of the diffeomorphism's Jacobian, this criterion can be made quantitative.
\end{enumerate}
In the absence of the first criterion, geodesic transport should be seen as superior to naive transport. In the absence of the second criterion, there is no good correspondence between conditions - as in the example of $F_0$ and $F_1$.

\subsection{Geodesic interpolation on the Olivetti dataset}

Here, the possible conditions are the 40 different persons in Olivetti. The dataset contains only 10 images for each condition, taken at different angles and with different lighting conditions, making it a challenging dataset for generative networks. Given the very small sample size here we used $F = [-1,1]$ as the standard fiber and $B = \R^{10}$. 

\begin{figure}[H]
\begin{center}
\includegraphics[width=0.90\textwidth]{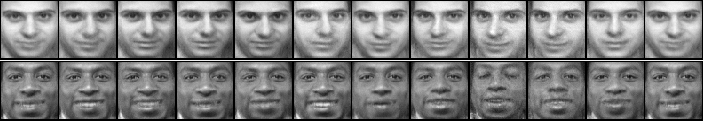}
\caption{Olivetti manifold plots. Images generated using an evenly spaced grid on the whole fiber space $F = [-1,1]$, for individuals 8 (top) and 21 (bottom).}
\label{fig:olivetti_manifold_plot}
\end{center}
\end{figure}

\begin{figure}[tbp]
\begin{center}
\includegraphics[width=0.90\textwidth]{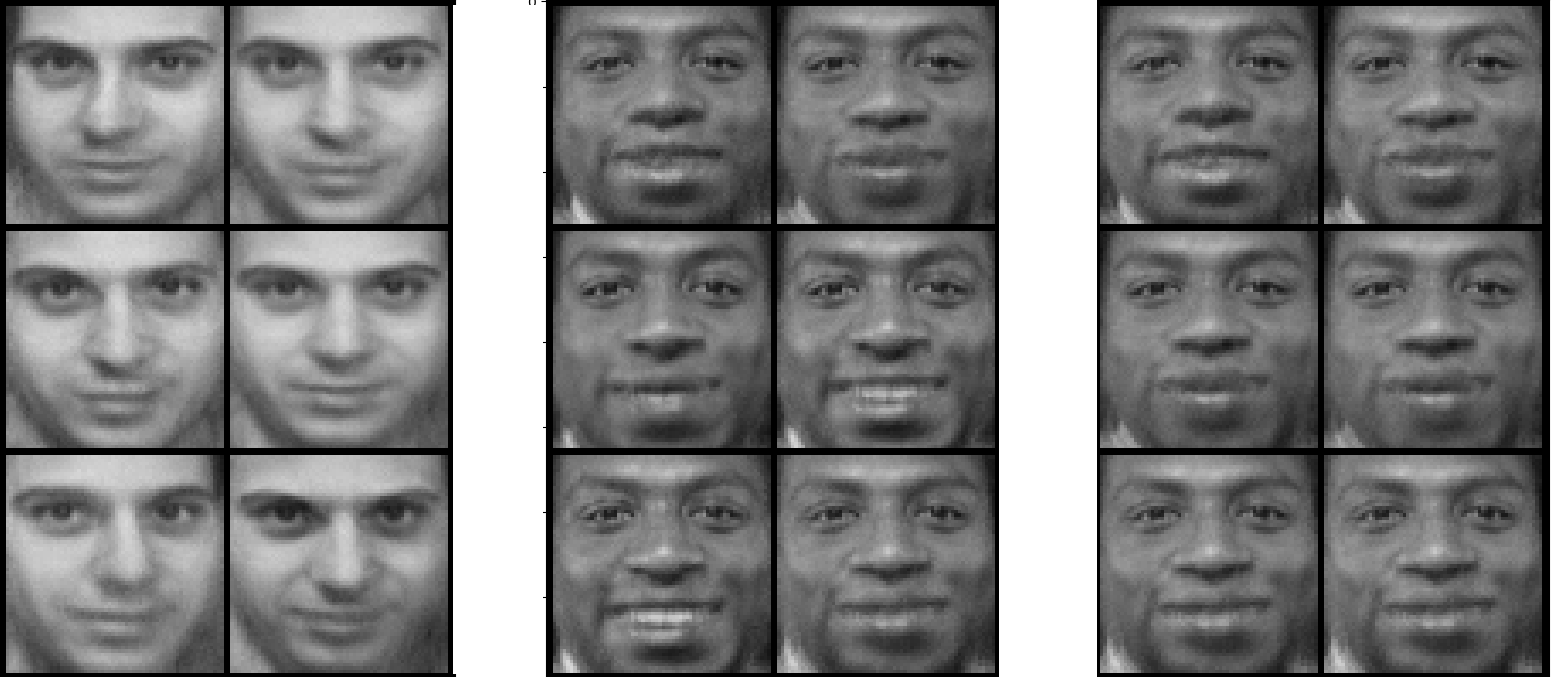}
\caption{Calculated correspondences from $F_8$ to $F_{21}$. Starting points on $F_8$ (left), correspondences calculated through naive transport (center), correspondences calculated through geodesic transport (right). Geodesic transport hides the teeth from subject 21.}
\label{fig:olivetti_correspondences}
\end{center}
\end{figure}

\begin{figure}[tbp]
\begin{center}
\includegraphics[width=0.90\textwidth]{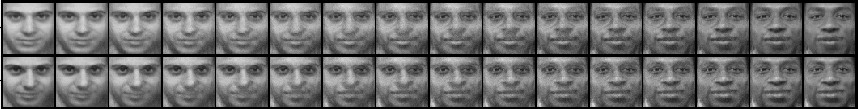}
\caption{Geodesic interpolation from $F_8$ to $F_{21}$. Individual $8$ gradually morphs into individual $21$.}
\label{fig:olivetti_interpolation}
\end{center}
\end{figure}

\begin{figure}[tbp]
\begin{center}
\includegraphics[width=1.0\textwidth]{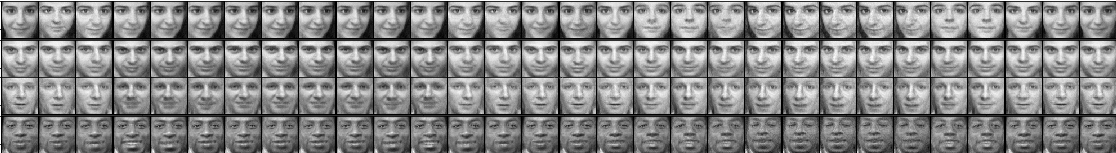}
\end{center}
\caption{Manifold plots of 30 generated images generated using an evenly spaced grid. Each person is a condition and $F = [-1, 1]$. The original dataset contains only 10 examples per person. We see that naive transport places face of similar angles at similar positions in $F$}
\label{fig:olivetti_faces}
\end{figure}

Fig. \ref{fig:olivetti_manifold_plot} shows generated samples for individuals 8 and 21. Despite the limited dataset size and the very small size of $F$, the network is able to generate more samples than in the training set. Individual 8, gradually turns his head, while individual 21 gradually smiles.
Fig. \ref{fig:olivetti_correspondences} shows the results of naive and geodesic transport from $F_8$ to $F_{21}$. Interestingly, geodesic transport created more accurate correspondence by hiding the teeth. Finally, Fig. \ref{fig:olivetti_interpolation} shows interpolations along the paths used to generate Fig. \ref{fig:olivetti_correspondences}. Individual 8 gradually transforms into individual 21. The network is able to generate faces from intermediary points which shows that it was able to generalize in both $B$ and $F$. Finally, we evaluated the smoothness of the learned fiber spaces, and the generalization capacity of FAEs by generating more samples than are present in the training dataset. Fig. \ref{fig:olivetti_faces}, shows 30 images for 4 individuals (3 times more that in the original dataset). This figure shows a smooth interpolation for every person, with faces at similar angles corresponding to similar coordinates in the fiber space.

\subsection{Evaluation on single-cell data}

Single-cell RNA sequencing measures the gene expression of each cell individually. The result is a matrix where each cell is represented by a vector of gene expressions.
However, differences in sample handling and technical platforms leave strong imprints that uniquely mark each batch and overshadow the biology. These imprints are commonly referred to as \emph{batch effects}. The process of \emph{batch correction} refers to the integration of batches together while preserving relevant biological signal, such as cell types that can be inferred through the expression of distinct gene modules.

\begin{figure*}[htp!]
\centering
\begin{adjustbox}{width=\linewidth}
\includegraphics[width=0.35\textwidth]{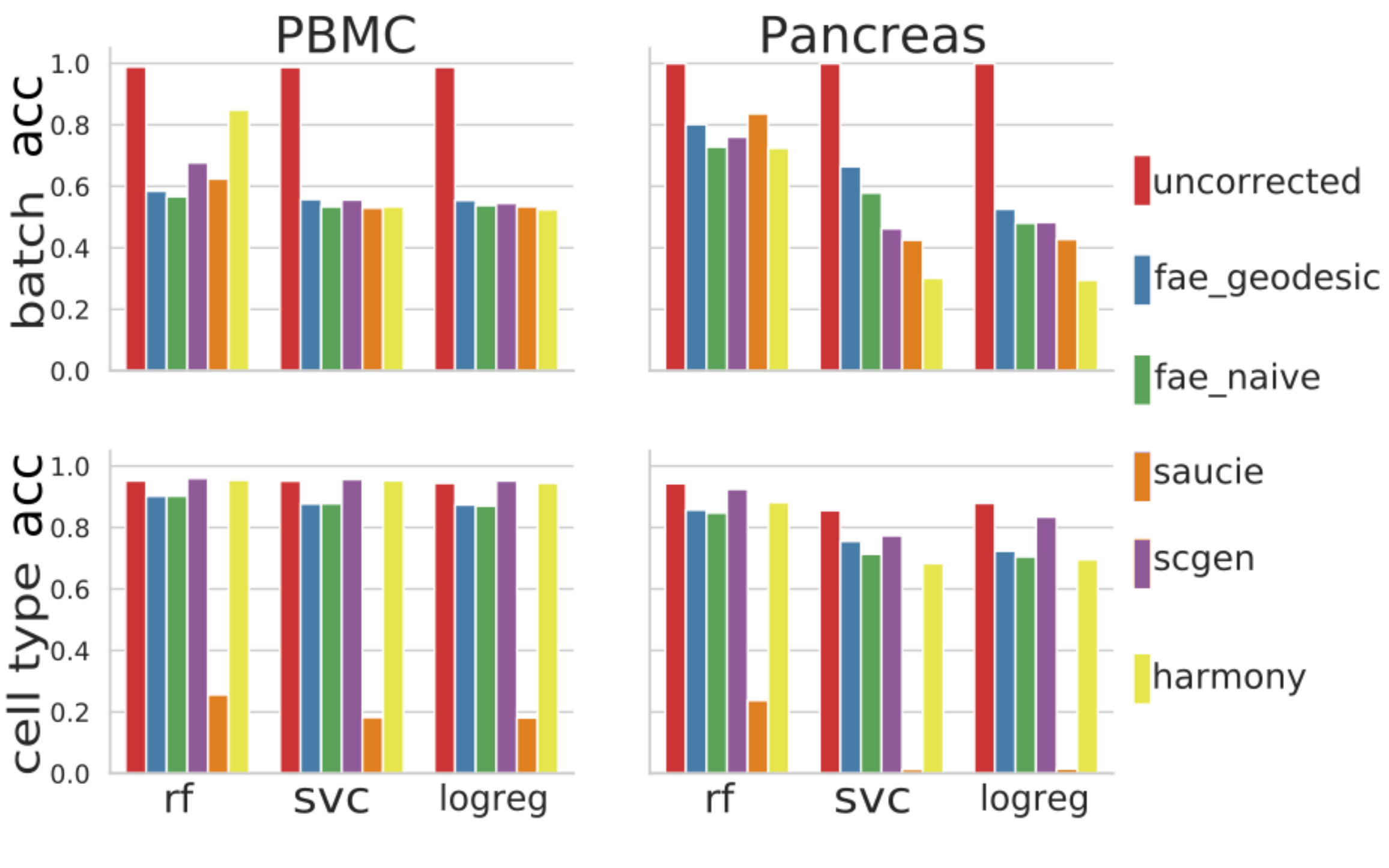}
\includegraphics[width=0.15\textwidth]{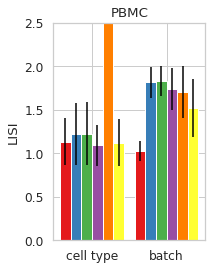}
\includegraphics[width=0.15\textwidth]{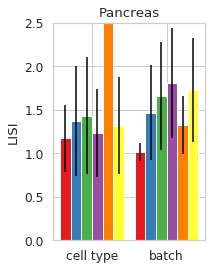}
\includegraphics[width=0.30\textwidth]{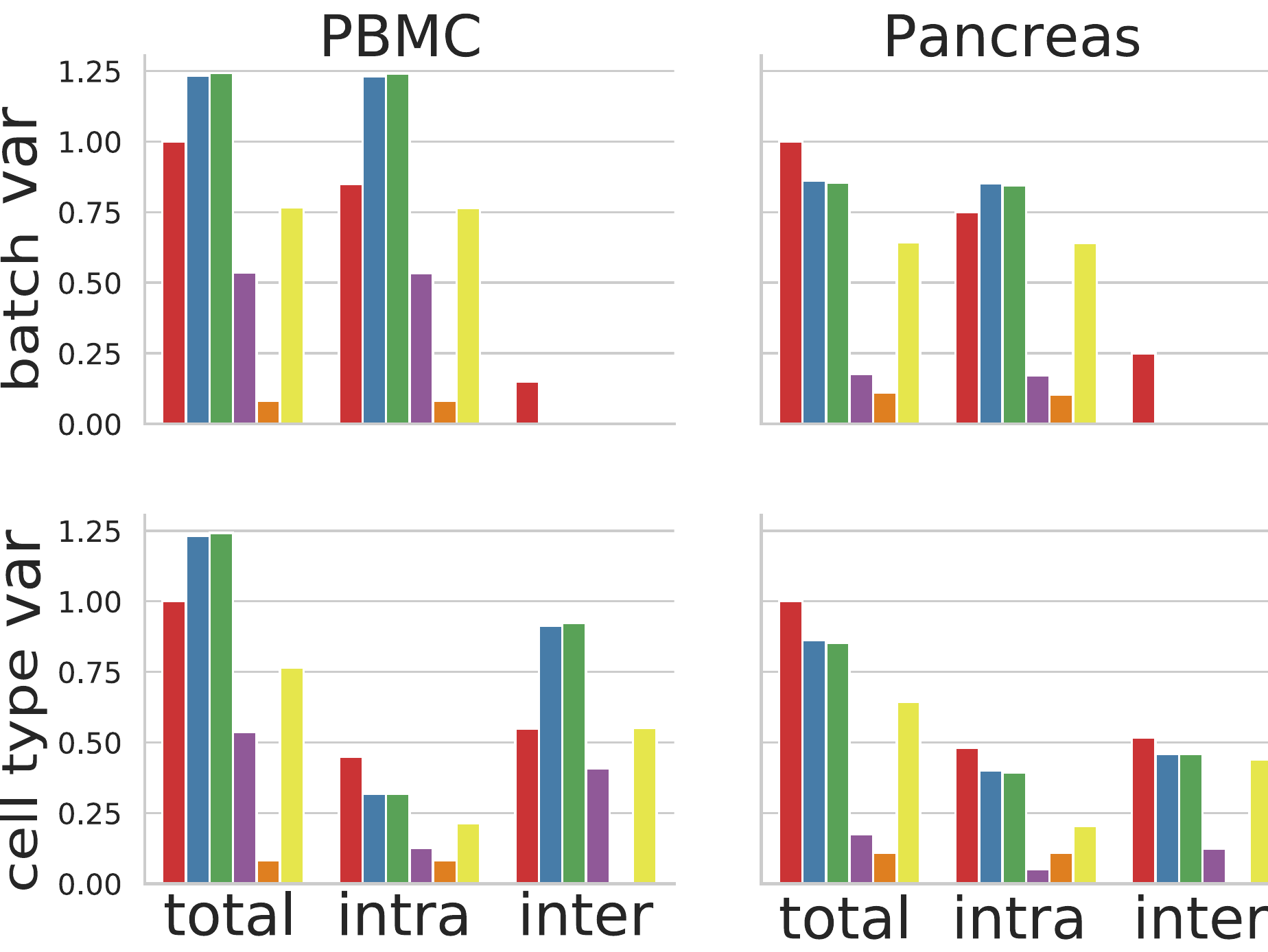}
\end{adjustbox}
\caption{Left: Accuracies for uncorrected and batch correction methods on both dataset. Batch accuracy (lower is better) and, cell type accuracy (higher is better) are reported for random forest (rf), support vector classifier (svc) and logistic regression (logreg).
Middle: LISI scores for uncorrected and batch correction methods on both datasets. LISI on cell type (closer to 1 is better), LISI on batch (higher is better). Error bars display the standard deviation.
Right: Total variance and Ward's variance decomposition, for uncorrected data and batch correction methods. Total variance has been normalized to 1, on the uncorrected dataset.
}
\label{fig:Barplots}
\end{figure*}

We use FAEs to represent batches as separate conditions, we then correct the batch effect by transporting all cells to a single reference batch. We benchmark our methods against the current state of the art in batch correction: Harmony \cite{HARMONY}, and two neural networks developed to handle and batch correct single-cell RNA sequencing data: scGen \cite{SCGEN} and SAUCIE \cite{SAUCIE}. As Harmony runs on principal components, we ran our benchmarks on PCA reduced data by using the first 20 principal components, in line with the methodology of \cite{BatchBenchmark}.

We quantify batch correction quality using the prediction accuracies of three classifiers and the batch correction metric LISI \cite{HARMONY}. 
Finally, we use Ward's variance decomposition to quantify how much changes in variance can be attributed to variations within groups, as opposed to in-between groups \cite{SAPORTA}[p.258]. LISI and Ward's method are detailed in supplementary material.
We benchmark all methods on two datasets. The first contains two batches of Peripheral Blood Mononuclear Cells (i.e., PBMCs): unstimulated and stimulated (with INF-b\footnote{Interferon-beta}) \cite{PBMC}. The second is a compilation of 4 published pancreatic datasets that have been generated by different groups using 4 distinct single-cell RNA sequencing experimental approaches \cite{PANCREAS}. This dataset is therefore the more challenging of the two.

A successful correction removes batch imprint and conserves cells biological identity. Therefore we report in Fig. \ref{fig:Barplots} the accuracy on predicting the batch (i.e., lower is better) and the accuracy on predicting the cell type (i.e., higher is better). All methods performed well when it comes to removing the batch signal. Despite using a bottle-neck about 10 times lower \footnote{We used 10 for pancreas and 16 for PBMC vs 100 for scGen.}, our transport method shows results close to to scGen \cite{SCGEN}. SAUCIE over-corrected the batch effect at the expense cell type identity. Compared to naive, geodesic transport increases both cell type and batch predictability. For the LISI scores, once again, naive and geodesic transport show results on par with scGen.

\begin{figure}[H]
\begin{center}
\includegraphics[trim=0 0 0 0 0, clip, width=0.8\textwidth]{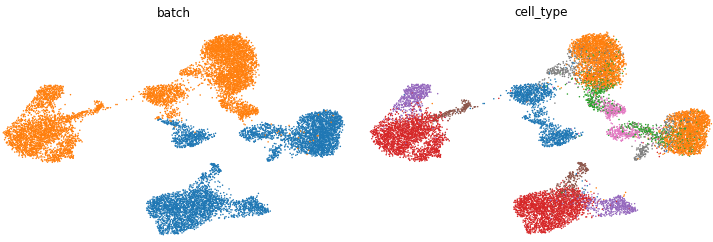} \\
\includegraphics[trim=0 0 0 0 0, clip, width=0.8\textwidth]{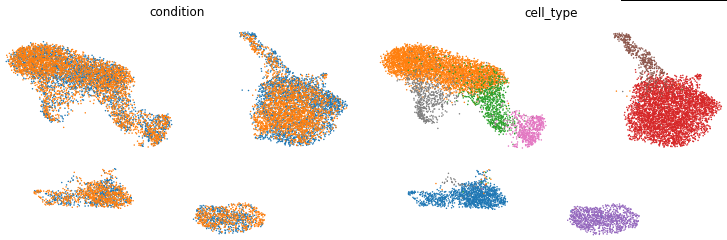} \\
\includegraphics[trim=0 0 0 0, clip, width=0.8\textwidth]{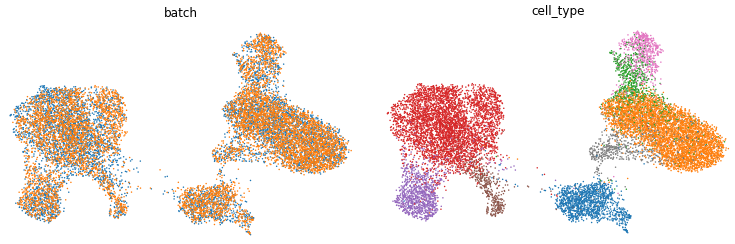}
\end{center}
\caption{UMAP visualization of PBMC cells. Left column: cells colored by batch, Right: colored by cell types. From top to bottom: uncorrected data, scGen, geodesic transport (the plot of naive transport is very close to the naked eye). Transport conserves cell types relationships by keeping purple cells (i.e., CD16 monocytes) close to red cells (i.e., CD14 monocytes), which are related to each other.}
\label{fig:umap_pbmc}
\end{figure}

\begin{figure}[!htp]
\begin{center}
\includegraphics[trim=0 0 0 0, clip, width=0.8\textwidth]{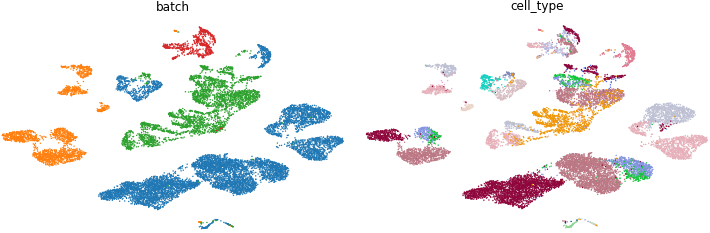}\\
\includegraphics[trim=0 0 0 0, clip, width=0.8\textwidth]{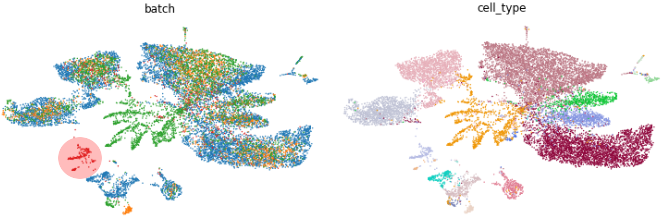}\\
\includegraphics[trim=0 0 0 0, clip, width=0.8\textwidth]{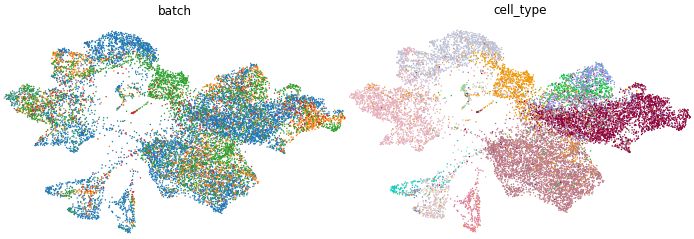}
\end{center}
\caption{UMAP visualization of pancreas cells. Left column: cells colored by batch, right: colored by cell types. From top to bottom: uncorrected data, scGen (bottleneck size: 100), naive transport (bottleneck size: 16). Contrary to scGen, naive transport was able to integrate cells form the red batch despite the small sample size. This suggests that FAE are better at integrating datasets of small sample sizes}
\label{fig:umap_pancreas}
\end{figure}

Ward's variance decomposition in Fig. \ref{fig:Barplots}, shows that transport is the only method capable of retaining a significant part of the original variance of the uncorrected dataset. This suggests that our method retains more of the biological signal. On the PBMC datasets, we observe increases in variance within batches and between cell types. Such increases could be due to the network imputing values for missing genes: due to single-cell experimental technical issues, gene expression matrices are very sparse (over 80\% to 90\%).


Fig. \ref{fig:umap_pbmc} shows UMAP plots \cite{UMAP} for uncorrected data, transport and scGen. After correction, batches overlap and cells cluster by cell types. In contrast to scGen, transport kept related cell types, CD16 monocytes (purple) and CD16 monocytes (red), close to each other as in the original uncorrected dataset. This suggest that FAEs are better at conserving relevant structures from datasets, and echoes variance decomposition results (Fig. \ref{fig:Barplots}). Finally, compared an FAE to scGen in the context of integrating batches of different natures, as every batch was generated using a different experimental approach, and different sizes. As shown in Fig. \ref{fig:umap_pancreas}, contrary to scGen, the  FAE was able to integrate cells from the the smallest batch (in red) with the rest of dataset. These results suggest that FAEs are more capable of handling datasets of small sample sizes. 


\section{Conclusion}
\label{section:conclusion}

In this work we proposed: 
(i) a geometric framework for representation learning which disentangles conditions from sample variations, 
(ii) a method for computing correspondences between conditions, 
(iii) a neural network architecture implementing the framework. 
We restricted ourselves to cases where conditions are of similar nature (e.g. datasets containing the same cell types). Consequently, naive and geodesic transport yield similar results. In turn, one could argue that this proximity morally measures similarity between conditions.

\medskip

Here are possible future directions of research:

\medskip

{\bf Regarding the geometric framework: } In this paper, we restricted ourselves to a Euclidean finite dimensional target space. The correct distance and the correct induced topology is in generally dependent on the domain of application. An interesting infinite dimensional example is given by the Wasserstein distances, which have very desirable properties in imaging. Efficient computations became possible since the introduction of the entropic regularization technique in \cite{C13} and the Riemannian geometry of the Wasserstein space is already on a solid mathematical footing via Otto calculus \cite{Villani_OT}[Chapter 15].

Also, we used in applications the hypercube $F=[-1, 1]^n$ as standard fiber. This is indeed simpler for practical implementation although the theory we just developed works better in a boundary free setting. As such it is natural to consider other manifolds and manifold-valued latent variables have recently been considered in \cite{DFDKT18} and \cite{XD18}.

\medskip

{\bf Regularization in high dimensions: }
We have found that the computation of geodesics becomes unstable in the case of high dimensional latent spaces, which could be referred to as a curse of dimensionality phenomenon. To solve the problem, we introduced a regularization cost so that instead of minimizing the energy functional $\Ec(\gamma)$ among admissible curves, we minimize a loss:
\begin{align}
\label{eq:lambda_reg}
   \Lc(\gamma) & = \Ec(\gamma) + \lambda_{\textrm{reg}}\left\| \gamma_{t=1} - \gamma^{\textrm{naive}}_{t=1} \right\|_2^2 \ .
\end{align}
Here $\gamma^{\textrm{naive}}_{t=1} = (f_1, b_2)$ is the result of naive transport. Thanks to the additional term, minimizing curves do not stray too far from the results given by naive transport. Other methods should be explored. In general, we strongly believe that the latent spaces of neural nets can and should be forced to have desirable features.


\bigskip

\bigskip

{\bf Acknowledgements:} 

We would like to thank Maude Dumont-Lagacé for her feedback and expert sketching skills in drawing Fig. \ref{fig:concept}.



\appendix

\newpage

Now, for the supplementary material where we will recall some statements of the paper and give more details, for reader's convenience. Also the code is available in an anonymous repository at:

\url{https://github.com/tariqdaouda/FiberedAE}

\section{Hyperparameters, datasets and training hardware}

\subsection{For the numerics of geodesics}

\

\medskip

{\bf Choices of hyper-parameters:} 
The parameters tracked in the automatic differentiation framework are the coefficients of $\gamma \in V_N$. For numerical experiments, we used $\Delta t = 1/256$ and $N=6$. The gradient descent used the classical RMSprop stepper in pyTorch.

A case-by-case description of learning rates and regularization parameters $\lambda_{\textrm{reg}}$ (see Eq. \eqref{eq:lambda_reg}) is available in JSON format:\\
\url{https://github.com/tariqdaouda/FiberedAE/tree/master/demos/geodesics_configurations}

\medskip

{\bf Training hardware:}
Computation of geodesics can be done in parallel for each geodesic separately, on independent copies of the trained neural network. In this context, we simply use multiple CPU threads and our machine of choice was a server with 192 cores Intel® Xeon® CPU E7-8890 v4 @ 2.20GHz and 512 GB of RAM. This is especially useful for batch correction since datasets come with tens of thousands cells and we needed one geodesic per cell. Otherwise, in order to compute geodesics in the hundreds as in Fig. \ref{fig:MNIST_geodesics}, a laptop is more than enough.

\subsection{For training FAEs}

All FAEs where trained on a single GPU on a Dell Precision 5530. RAM: 31GiB, CPU: Intel(R) Core(TM) i7-8850H CPU @ 2.60GHz, GPU: Quadro P2000 Mobile.

\subsection{Datasets}
\subsubsection{Image datasets}
We used the version of MNIST \cite{MNIST} provided by pyTorch and the version of Olivetti provided by scikit-learn \cite{SKLEARN}.

\subsubsection{Single cell datasets pre-processing}
We used the version of the pancreas dataset provided by \cite{SCGEN} dataset size is $(14,693, 2448)$. For the PBMC dataset \cite{PBMC}, we downloaded the pre-normalized version of the dataset provided by \cite{TRVAE}. We normalized gene expressions per cell, transformed expressions using $log(expression +1)$, and used only the top $2000$ highly variable genes after processing dataset size was: $(13576, 2000)$.

For FAE training, all expressions for all datasets where normalized in $[0, 1]$ by substracting the minimum value and dividing by the maximum value. As Harmony runs on principal components, we run our benchmarks on PCA reduced data by using the first 20 principal components, in line with the methodology introduced by \cite{BatchBenchmark}.

\section{Image datasets experiments}

All FAEs hyper-parameters where optimized using grid search, with reconstruction learning rates at least 10 times bigger than the highest second learning rate. We considered the following:
\begin{itemize}
    \item Learning rates: ($10^{-3}$, $2.10^{-3}$, $10^{-4}$, $2.10^{-4}$, $10^{-5}$, $2.10^{-5}$, $10^{-6}$, $2.10^{-6}$).
    \item Hidden layer sizes: $(64, 128, 256)$.
    \item Bottleneck (fiber) sizes: $(1, 2, 5,  8, 10, 16, 32)$.
    \item Depths: $(10, 20, 30, 40, 50, 80, 100)$.
    \item Embedding sizes ($\dim B$): $(1, 2, 10)$.
\end{itemize}

\section{Single-cell RNA sequencing experiments}

\subsection{Classifiers training}

We used three classifiers form sklearn v0.22 \cite{SKLEARN}:
Logistic regression (\emph{LogisticRegression}, with solver \emph{lbfgs}), Suport vector classifier (\emph{SVC} with kernel \emph{rbf}), random forest (\emph{RandomForestClassifier} with $500$ estimators). For all other hyper-parameters we used the defaults of sklearn v0.22.

All classifier were trained with class balancing and were tested on a randomly selected test set containing 25\% of the total dataset.

\subsection{FAE hyper-parameters selection}
All FAEs hyper-parameters where optimized using grid search, with reconstruction learning rates at least 10 times bigger than the highest second learning rate. We considered the following:

\begin{itemize}
    \item Learning rates: ($10^{-3}$, $2.10^{-3}$, $10^{-4}$, $2.10^{-4}$, $10^{-5}$, $2.10^{-5}$, $10^{-6}$, $2.10^{-6}$).
    \item Hidden layer sizes: $(64, 128, 256)$.
    \item Bottleneck sizes: $(1, 2, 5,  8, 10, 16, 32)$.
    \item Depths: $(10, 20, 30, 40, 50, 80, 100)$.
    \item Embedding sizes ($\dim B$): $2$ for pancreas and $1$ for PBMC.
\end{itemize}

\subsection{FAE Architectures}

All architectures are available in JSON format at:\\
\url{https://github.com/tariqdaouda/FiberedAE/tree/master/demos/configurations}

\subsection{scGen architecture}
We used the following hyper-parameters: 
\begin{itemize}
    \item z\_dimension: 100
    \item learning\_rate: 0.001
    \item dropout\_rate: 0.2
    \item alpha: 0.00005
\end{itemize}

\subsection{SAUCIE architecture}
We trained this model using the following hyper-parameters: 
\begin{itemize}
    \item lambda\_b=$0.1$,
    \item lambda\_c=$0$,
    \item layer\_c=$0$,
    \item lambda\_d=$0$,
    \item layers=$[512,256,128,2]$,
    \item activation=$ReLU$,
    \item learning\_rate=$0.001$
\end{itemize}

\subsection{Metric: LISI (Local Inverse Simpson Index)}

LISI measures the inverse Simpson Index in the local neighbourhood of each cell. Because it is a per-cell measure, here we reported the aggregated mean and standard deviation metrics over all cells. We used the implementation of \cite{HARMONY}.

For a given cell its LISI score is computed as follows. Let $S$ be the inverse Simpson index: $$ S= \frac{1}{\sum_{i=1}^Z p(i)} $$ where $Z$ is the number of batches in the dataset, and $p(i)$ the probability of the batch $i$ being present in the local neighbourhood of that cell. To define the neighbourhood of cells, probabilities are computed using a gaussian kernel centered over each cells (here we used a perplexity of $30$ as suggested by \cite{HARMONY}).The resulting $S$ is the average number of cells needed to be sampled before two cells from the cell batch are drawn. In a unmixed dataset the LISI score for most cells is close to $1$, as it only takes on average $1$ draw to a get cell from the same batch. In a well mixed dataset the value should be close to the total number of batches. 

Just like \cite{HARMONY} we used two versions of the LISI score. One on the batch, and one on the cell-type. Both score are computed in the exact same way, they are however interpreted differently. As the LISI score on batches should increase to become closer to the number of batches, the LISI score on cell-types should remain close to $1$. This is because cells of the same types should remain clustered together after correction. An increase in the LISI score on cell-types indicates a loss of cell-type specific features and therefore of relevant biological information.

\subsection{Metric: Huygens-Ward variance decomposition}
Consider a sample
$$\left\{ X_1, \dots, X_N \right\} \subset \R^n$$
which can be clustered into $k$ groups $\left( G_i \ ; \ 1 \leq i \leq k \right)$. By definition, the global sample mean and variance are:
\begin{align*}
 m        := \frac{1}{N} \sum_{i=1}^N X_i \ , & \quad \quad
 \sigma^2 := \frac{1}{N} \sum_{i=1}^N \left\| X_i - m \right\|_2^2 \ .
\end{align*}

On the one hand, assimilating each group $G_j$ to its center of mass
$$ m_j := \frac{1}{|G_j|} \sum_{X_i \in G_j} X_i \ , $$
the inter-class variance is defined as:
$$ \sigma^2_e := \frac{1}{k} \sum_{j=1}^k \left\| m_j - m \right\|^2 \ .$$

On the other hand, considering each group separately, intra-class variances for each group $G_j$ are given by:
\begin{align*}
 \sigma^2_j & := \frac{1}{|G_j|} \sum_{X_i \in G_j} \left\| X_i - m_j \right\|_2^2 \ .
\end{align*}
Here $|G_j|$ stands for the cardinal of the group $G_j$. The intra-class variance is defined as:
$$
    \sigma^2_a := \frac{1}{k} \sum_{j=1}^k \sigma^2_j \ .
$$

In the end, performing a conditioning with respect to the groups, we invoke the law of total variance. We obtain that total variance decomposes into the sum of inter-class variance and intra-class variance:
\begin{align}
\label{eq:huygens_ward}
\sigma^2 & := \sigma^2_e + \sigma^2_a \ .
\end{align}
This is also known as the Huygens-Ward decomposition of the variance. Such a decomposition has a very clear statistical interpretation for the purposes of batch correction. Grouping the cells by batch, the goal is to minimize the inter-class variance while not destroying the intra-class variance. Grouping the cells by cell-type, the goal is maximizing cell separability by increasing inter-class variance.

\bibliographystyle{halpha}
\bibliography{FAE.bib}

\end{document}

%% file: figures/figFiberBundle.tex
\begin{center}
\begin{tikzpicture}[>=stealth]
\draw(3,0)--(5,0)node[midway,below]{Base manifold $B$}arc(0:70:4)--(90:3)arc(90:0:3)--cycle;


\begin{scope}[bend right]
\foreach \i[count=\x] in {10,30,50,70}
{\node(a\x)[circle,fill,inner sep=1pt]at (\i:3.4){};
\draw(a\x)to(a\x|-1,4);}

\foreach \i[count=\x] in {7,26,46,49}
{\node(b\x)[circle,fill,inner sep=1pt]at (\i:4){};
\draw(b\x)to(b\x|-1,4);}

\foreach \i[count=\x] in {6,26,46,66}
{\node(c\x)[circle,fill,inner sep=1pt]at (\i:3.6){};
\draw(c\x)to(c\x|-0,4);}

\path(c1)to coordinate[near start](d)(c1|-0,4);
\path(b1)to coordinate[near start](e)(b1|-0,4);
\end{scope}

\draw[<-](d)--+(0.5,-0.5)node(dl)[right]{Fibers $\approx F$};
\draw[<-](e)--+(0.25,-0.5)node[right]{};
\end{tikzpicture}
\caption{Illustrative drawing of fiber bundle $B$ and standard fiber $F$}
\label{fig:fiberBundle}
\end{center}

%% file: figures/figGeodesicTransport.tex
\begin{center}
\begin{tikzpicture}[scale=0.9, >=stealth]
\draw(1,0)--(8,0)node[midway,below]{Base $B$}arc(0:70:4)--(90:1)arc(90:0:1)--cycle;

\node(b1)[circle,fill,inner sep=1pt]at (2,1.4){};
\draw[bend right](b1)node[below]{$b_1$} to node(v)[left, pos=0.5]{$f_1$} (b1|-1,4) node[right]{Fiber $F_1$};

\node(b2)[circle,fill,inner sep=1pt]at (6,0.5){};
\draw[bend right](b2)node[right]{$b_2$} to[out=0] node(w)[right, pos=0.33]{$f_2$} (b2|-1,4) node[right]{Fiber $F_2$};


\draw[->, bend right](v) to[out=60, in=180] node[midway, above]{$\gamma$} (w);
\end{tikzpicture}
\caption{Illustration of geodesic transport with curve $\gamma$ from the fiber $F_1$ to $F_2$.}
\label{fig:transport}
\end{center}

%% file: figures/figNeuralNetworkStructure.tex
\tikzset{%
  >={Latex[width=2mm,length=2mm]},
           input/.style = {rectangle, draw=black,
                           minimum height=1cm,
                           text centered, font=\sffamily},
            base/.style = {rectangle, rounded corners, draw=black,
                           minimum width=4cm, minimum height=1cm,
                           text centered, font=\sffamily},
       optimizer/.style = {circle, text centered, fill=red!30},
  activityStarts/.style = {base, fill=blue!30},
    activityRuns/.style = {base, fill=green!30},
         process/.style = {base, minimum width=2.5cm, fill=orange!15,
                           font=\ttfamily},
}
\usetikzlibrary{arrows.meta}
\begin{figure*}[htp!]
\scalebox{0.95}{

\begin{adjustbox}{width=\linewidth}
\begin{tikzpicture}[scale=0.1, node distance=1.5cm,
    every node/.style={fill=white, font=\sffamily}, align=center]
  \node (legendModule) [process, xshift=-4cm, yshift=5cm] {Network, Parameters};
  \node (legendModuleText) [right of=legendModule, xshift=1.6cm] {Modules};
  
  \node (legendOpt) [optimizer, xshift=-5cm, yshift=3.5cm] {Loss};
  \node (legendOptText) [right of=legendOpt, xshift=0.5cm] {Optimizers};
  
\node (legendVar) [input, xshift=-4cm, yshift=2cm] {Input / output};
  
  \node (start)             [xshift=-2cm]     							 {};

  \node (data)              [input, left  of=start]         {\raisebox{12pt}{$X=\ $}\includegraphics[scale=0.5]{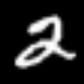}};
  \node (dataText)          [left  of=data, xshift=-1.0cm]  {Data in $\Xc$};
  
  \node (label)             [input, below of=data]           {$c = 2$};
  \node (labelText)         [left of=label, xshift=-0.25cm]  {Condition};
  
  \node (encoder)        [process, right of=data  , xshift=2.0cm]     {Encoder $\Phi_1$, $\theta_e$};
  \node (f)              [input, right of=encoder , xshift=1.5cm]     {$f \in F$};
    
  \node (embeddings)     [process, right of=label,    xshift=2.0cm]   {Embeddings $\Phi_2$, $\theta_m$};
  \node (b)              [input, right of=embeddings, xshift=1.5cm]   {$b \in B$};
  
  \node (ACBlock)        [process,   above of=f,      yshift=0.5cm]  {Classifier with GRL $\overline{\Upsilon}$, $\theta_{ac}$};
  \node (AC_out)         [input,     above of=ACBlock,yshift=0.5cm]  {$\left( \log \P\left( \widehat{c} = c_i \right) \right)_{1 \leq i \leq K}$};
  \node (ACLoss)         [optimizer, above of=AC_out, yshift=0.5cm]  {$\Lc^{xent}$ \\ Eq.\eqref{def:xent}};
  
  \node (concatData)     [right of=encoder, yshift=-0.75cm, xshift=2cm]           {};
  \node (latent)         [input, right of=encoder, yshift=-0.75cm, xshift=4cm]    {$\left(b, f\right) \in M$};
  \node (concatDataText) [above of=latent, yshift=-1cm]                           {Latent variable};

  \node (decoder)        [process, right of=latent, xshift=2cm]    {Decoder $\Psi$, $\theta_d$};
  
  \node (Xhat)           [input,  right of=decoder, xshift=2cm]   {\raisebox{12pt}{$\widehat{X}=\ $}\includegraphics[scale=0.5]{./figures/MnistDigit2.png}};
  
  \node (XhatLoss)       [optimizer, right of=Xhat, xshift=2cm]   {$\Lc^{mse}$ \\ Eq.\eqref{def:MSE}};
  
  \node (discriminator)  [process, above of=Xhat, xshift=-2.5cm, yshift=1cm]    {Discriminator $\Delta$, $\theta_\Delta$};
  \node (classifier)     [process, above of=Xhat, xshift= 2.5cm, yshift=1cm]    {Classifier $\Upsilon$, $\theta_c$};
  
  \node (D_out)          [input, above of=discriminator, yshift=0.5cm]    {$\log \P\left( \widehat{X} \  \textrm{is fake} \right)$};
  \node (C_out)          [input, above of=classifier   , yshift=0.5cm]    {$\left( \log \P\left( \widehat{c} = c_i \right) \right)_{1 \leq i \leq K}$};
  
  \node (Loss1)      [optimizer, above of=D_out, yshift=0.5cm] {$\Lc^{xent}$ \\ Eq.\eqref{def:xent}};
  \node (Loss2)      [optimizer, above of=C_out, yshift=0.5cm] {$\Lc^{xent}$ \\
  Eq.\eqref{def:xent}};

  \draw[->]     (data)       -- (encoder);
  \draw[->]     (encoder)    -- (f);
  \draw[->]     (f)          -- (ACBlock);
  \draw[->]     (ACBlock)    -- (AC_out);
  \draw[->]     (AC_out)     -- (ACLoss);

  \draw[->]     (f.east) .. controls +(10,0) and +(-10,0) ..  (latent.west);
  
  \draw[->]     (label)      -- (embeddings);
  \draw[->]     (embeddings) -- (b);
  \draw[->]     (b.east)     .. controls +(10,0) and +(-10,0) ..  (latent.west);
  
  \draw[->]     (latent) -- (decoder);

  \draw[->]     (decoder) -- (Xhat);

  \draw[->]     (Xhat) -- (XhatLoss);
  \draw[->]     (Xhat.north) -- (discriminator.south);
  \draw[->]     (Xhat.north) -- (classifier.south);

  \draw[->]     (discriminator) -- (D_out);
  \draw[->]     (D_out)         -- (Loss1);
  \draw[->]     (classifier) -- (C_out);
  \draw[->]     (C_out)      -- (Loss2);

\end{tikzpicture}
\end{adjustbox}
}
\caption{Network architecture. The general architecture is that of an auto-encoder receiving couples of samples and conditions $(X, c)$ and outputting a reconstruction $\widehat{X}$. 
The latent space is stratified into the fiber coordinate $f$ (output of the bottleneck layer), and the base coordinate $b$ encoding conditions. To the auto-encoder architecture we have added the classifier $\overline{\Upsilon}$ coupled with a GRL to disentangle $f$ from $b$, the GAN discriminator $\Delta$ to ensure reconstruction realism, and the condition classifier $\Upsilon$ to prevent mode collapses.}
\label{fig:nnArchitecture}
\end{figure*}

%% file: FAE_arxiv.bbl
\newcommand{\etalchar}[1]{$^{#1}$}
\begin{thebibliography}{GPAM{\etalchar{+}}14}

\bibitem[AAB{\etalchar{+}}15]{TF}
Mart\'{\i}n Abadi, Ashish Agarwal, Paul Barham, Eugene Brevdo, Zhifeng Chen,
  Craig Citro, Greg~S. Corrado, Andy Davis, Jeffrey Dean, Matthieu Devin,
  Sanjay Ghemawat, Ian Goodfellow, Andrew Harp, Geoffrey Irving, Michael Isard,
  Yangqing Jia, Rafal Jozefowicz, Lukasz Kaiser, Manjunath Kudlur, Josh
  Levenberg, Dan Man\'{e}, Rajat Monga, Sherry Moore, Derek Murray, Chris Olah,
  Mike Schuster, Jonathon Shlens, Benoit Steiner, Ilya Sutskever, Kunal Talwar,
  Paul Tucker, Vincent Vanhoucke, Vijay Vasudevan, Fernanda Vi\'{e}gas, Oriol
  Vinyals, Pete Warden, Martin Wattenberg, Martin Wicke, Yuan Yu, and Xiaoqiang
  Zheng.
\newblock {TensorFlow}: Large-scale machine learning on heterogeneous systems,
  2015.
\newblock Software available from tensorflow.org.

\bibitem[ACB17]{WGAN}
Martin Arjovsky, Soumith Chintala, and L{\'e}on Bottou.
\newblock Wasserstein generative adversarial networks.
\newblock In {\em International conference on machine learning}, pages
  214--223, 2017.

\bibitem[AVDS{\etalchar{+}}19]{SAUCIE}
Matthew Amodio, David Van~Dijk, Krishnan Srinivasan, William~S Chen, Hussein
  Mohsen, Kevin~R Moon, Allison Campbell, Yujiao Zhao, Xiaomei Wang, Manjunatha
  Venkataswamy, et~al.
\newblock Exploring single-cell data with deep multitasking neural networks.
\newblock {\em Nature methods}, pages 1--7, 2019.

\bibitem[BBR{\etalchar{+}}18]{MINE}
Mohamed~Ishmael Belghazi, Aristide Baratin, Sai Rajeshwar, Sherjil Ozair,
  Yoshua Bengio, Aaron Courville, and Devon Hjelm.
\newblock Mutual information neural estimation.
\newblock In {\em International Conference on Machine Learning}, pages
  531--540, 2018.

\bibitem[BLPL07]{AE1}
Yoshua Bengio, Pascal Lamblin, Dan Popovici, and Hugo Larochelle.
\newblock Greedy layer-wise training of deep networks.
\newblock In {\em Advances in neural information processing systems}, pages
  153--160, 2007.

\bibitem[CLGD18]{BetaTCVAE}
Ricky~TQ Chen, Xuechen Li, Roger~B Grosse, and David~K Duvenaud.
\newblock Isolating sources of disentanglement in variational autoencoders.
\newblock In {\em Advances in Neural Information Processing Systems}, pages
  2610--2620, 2018.

\bibitem[Cut13]{C13}
Marco Cuturi.
\newblock Sinkhorn distances: Lightspeed computation of optimal transport.
\newblock In {\em Advances in neural information processing systems}, pages
  2292--2300, 2013.

\bibitem[Dac07]{DACOROGNA}
Bernard Dacorogna.
\newblock {\em Direct methods in the calculus of variations}, volume~78.
\newblock Springer Science \& Business Media, 2007.

\bibitem[DFC{\etalchar{+}}18]{DFDKT18}
Tim~R. Davidson, Luca Falorsi, Nicola~De Cao, Thomas Kipf, and Jakub~M.
  Tomczak.
\newblock Hyperspherical variational auto-encoders.
\newblock In Amir Globerson and Ricardo Silva, editors, {\em Proceedings of the
  Thirty-Fourth Conference on Uncertainty in Artificial Intelligence, {UAI}
  2018, Monterey, California, USA, August 6-10, 2018}, pages 856--865. {AUAI}
  Press, 2018.

\bibitem[DZPL18]{HGN}
Tariq Daouda, Jeremie Zumer, Claude Perreault, and S{\'e}bastien Lemieux.
\newblock Holographic neural architectures.
\newblock {\em arXiv preprint arXiv:1806.00931}, 2018.

\bibitem[GHL90]{GHL90}
Sylvestre Gallot, Dominique Hulin, and Jacques Lafontaine.
\newblock {\em Riemannian geometry}, volume~2.
\newblock Springer, 1990.

\bibitem[GL15]{DANN1}
Yaroslav Ganin and Victor Lempitsky.
\newblock Unsupervised domain adaptation by backpropagation.
\newblock In {\em International conference on machine learning}, pages
  1180--1189. PMLR, 2015.

\bibitem[GPAM{\etalchar{+}}14]{GAN}
Ian Goodfellow, Jean Pouget-Abadie, Mehdi Mirza, Bing Xu, David Warde-Farley,
  Sherjil Ozair, Aaron Courville, and Yoshua Bengio.
\newblock Generative adversarial nets.
\newblock In {\em Advances in neural information processing systems}, pages
  2672--2680, 2014.

\bibitem[GSSG12]{GSFG12}
Boqing Gong, Yuan Shi, Fei Sha, and Kristen Grauman.
\newblock Geodesic flow kernel for unsupervised domain adaptation.
\newblock In {\em 2012 IEEE Conference on Computer Vision and Pattern
  Recognition}, pages 2066--2073. IEEE, 2012.

\bibitem[GUA{\etalchar{+}}17]{DANN2}
Yaroslav Ganin, Evgeniya Ustinova, Hana Ajakan, Pascal Germain, Hugo
  Larochelle, Fran{\c{c}}ois Laviolette, Mario Marchand, and Victor Lempitsky.
\newblock Domain-adversarial training of neural networks.
\newblock In {\em Domain Adaptation in Computer Vision Applications}, pages
  189--209. Springer, 2017.

\bibitem[HMP{\etalchar{+}}16]{BetaVAE}
Irina Higgins, Loic Matthey, Arka Pal, Christopher Burgess, Xavier Glorot,
  Matthew Botvinick, Shakir Mohamed, and Alexander Lerchner.
\newblock beta-vae: Learning basic visual concepts with a constrained
  variational framework.
\newblock 2016.

\bibitem[HSM{\etalchar{+}}17]{SCAN}
Irina Higgins, Nicolas Sonnerat, Loic Matthey, Arka Pal, Christopher~P Burgess,
  Matko Bosnjak, Murray Shanahan, Matthew Botvinick, Demis Hassabis, and
  Alexander Lerchner.
\newblock Scan: Learning hierarchical compositional visual concepts.
\newblock {\em arXiv preprint arXiv:1707.03389}, 2017.

\bibitem[IT01]{IT01}
Jin-ichi Itoh and Minoru Tanaka.
\newblock The lipschitz continuity of the distance function to the cut locus.
\newblock {\em Transactions of the American Mathematical Society},
  353(1):21--40, 2001.

\bibitem[KB14]{ADAM}
Diederik~P Kingma and Jimmy Ba.
\newblock Adam: A method for stochastic optimization.
\newblock {\em arXiv preprint arXiv:1412.6980}, 2014.

\bibitem[KLA19]{SGAN}
Tero Karras, Samuli Laine, and Timo Aila.
\newblock A style-based generator architecture for generative adversarial
  networks.
\newblock In {\em Proceedings of the IEEE Conference on Computer Vision and
  Pattern Recognition}, pages 4401--4410, 2019.

\bibitem[KM18]{FactorVAE}
Hyunjik Kim and Andriy Mnih.
\newblock Disentangling by factorising.
\newblock {\em arXiv preprint arXiv:1802.05983}, 2018.

\bibitem[KMF{\etalchar{+}}19]{HARMONY}
Ilya Korsunsky, Nghia Millard, Jean Fan, Kamil Slowikowski, Fan Zhang, Kevin
  Wei, Yuriy Baglaenko, Michael Brenner, Po-ru Loh, and Soumya Raychaudhuri.
\newblock Fast, sensitive and accurate integration of single-cell data with
  harmony.
\newblock {\em Nature methods}, pages 1--8, 2019.

\bibitem[KMRW14]{CVAE}
Durk~P Kingma, Shakir Mohamed, Danilo~Jimenez Rezende, and Max Welling.
\newblock Semi-supervised learning with deep generative models.
\newblock In {\em Advances in neural information processing systems}, pages
  3581--3589, 2014.

\bibitem[KSB17]{DIPVAE}
Abhishek Kumar, Prasanna Sattigeri, and Avinash Balakrishnan.
\newblock Variational inference of disentangled latent concepts from unlabeled
  observations.
\newblock {\em arXiv preprint arXiv:1711.00848}, 2017.

\bibitem[KSM99]{KSM99}
Ivan Kol{\'a}r, Jan Slov{\'a}k, and Peter~W Michor.
\newblock Natural operations in differential geometry.
\newblock 1999.

\bibitem[KST{\etalchar{+}}18]{PBMC}
Hyun~Min Kang, Meena Subramaniam, Sasha Targ, Michelle Nguyen, Lenka Maliskova,
  Elizabeth McCarthy, Eunice Wan, Simon Wong, Lauren Byrnes, Cristina~M Lanata,
  et~al.
\newblock Multiplexed droplet single-cell rna-sequencing using natural genetic
  variation.
\newblock {\em Nature biotechnology}, 36(1):89, 2018.

\bibitem[LBL{\etalchar{+}}19]{locatello2019challenging}
Francesco Locatello, Stefan Bauer, Mario Lucic, Gunnar Raetsch, Sylvain Gelly,
  Bernhard Sch{\"o}lkopf, and Olivier Bachem.
\newblock Challenging common assumptions in the unsupervised learning of
  disentangled representations.
\newblock In {\em international conference on machine learning}, pages
  4114--4124, 2019.

\bibitem[LC10]{MNIST}
Yann LeCun and Corinna Cortes.
\newblock {MNIST} handwritten digit database.
\newblock 2010.

\bibitem[LCC{\etalchar{+}}17]{MMD_GAN}
Chun-Liang Li, Wei-Cheng Chang, Yu~Cheng, Yiming Yang, and Barnab{\'a}s
  P{\'o}czos.
\newblock Mmd gan: Towards deeper understanding of moment matching network.
\newblock In {\em Advances in Neural Information Processing Systems}, pages
  2203--2213, 2017.

\bibitem[LNTW19]{TRVAE}
Mohammad Lotfollahi, Mohsen Naghipourfar, Fabian~J Theis, and F~Alexander Wolf.
\newblock Conditional out-of-sample generation for unpaired data using trvae.
\newblock {\em arXiv preprint arXiv:1910.01791}, 2019.

\bibitem[LSL{\etalchar{+}}15]{FAIRVAE}
Christos Louizos, Kevin Swersky, Yujia Li, Max Welling, and Richard Zemel.
\newblock The variational fair autoencoder.
\newblock {\em arXiv preprint arXiv:1511.00830}, 2015.

\bibitem[LSLW16]{VAE_GAN}
Anders Boesen~Lindbo Larsen, Soren~Kaae Sonderby, Hugo Larochelle, and Ole
  Winther.
\newblock Autoencoding beyond pixels using a learned similarity metric.
\newblock In {\em International conference on machine learning}, pages
  1558--1566. PMLR, 2016.

\bibitem[LWT19]{SCGEN}
Mohammad Lotfollahi, F~Alexander Wolf, and Fabian~J Theis.
\newblock scgen predicts single-cell perturbation responses.
\newblock {\em Nature methods}, 16(8):715--721, 2019.

\bibitem[MHM18]{UMAP}
Leland McInnes, John Healy, and James Melville.
\newblock Umap: Uniform manifold approximation and projection for dimension
  reduction.
\newblock {\em arXiv preprint arXiv:1802.03426}, 2018.

\bibitem[MO14]{CGAN}
Mehdi Mirza and Simon Osindero.
\newblock Conditional generative adversarial nets.
\newblock {\em arXiv preprint arXiv:1411.1784}, 2014.

\bibitem[MSJ{\etalchar{+}}15]{AAN}
Alireza Makhzani, Jonathon Shlens, Navdeep Jaitly, Ian Goodfellow, and Brendan
  Frey.
\newblock Adversarial autoencoders.
\newblock {\em arXiv preprint arXiv:1511.05644}, 2015.

\bibitem[{Nic}19]{Hauberg}
{Nicki Skafte Detlefsen and Soren Hauberg}.
\newblock Explicit disentanglement of appearance and perspective in generative
  models.
\newblock In {\em Advances in Neural Information Processing Systems (NeurIPS)},
  2019.

\bibitem[Ode16]{SSGAN}
Augustus Odena.
\newblock Semi-supervised learning with generative adversarial networks.
\newblock {\em arXiv preprint arXiv:1606.01583}, 2016.

\bibitem[PGC{\etalchar{+}}17]{PYTORCH}
Adam Paszke, Sam Gross, Soumith Chintala, Gregory Chanan, Edward Yang, Zachary
  DeVito, Zeming Lin, Alban Desmaison, Luca Antiga, and Adam Lerer.
\newblock Automatic differentiation in pytorch.
\newblock 2017.

\bibitem[PPMT18]{PANCREAS}
Jong-Eun Park, Krzysztof Pola{\'n}ski, Kerstin Meyer, and Sarah~A Teichmann.
\newblock Fast batch alignment of single cell transcriptomes unifies multiple
  mouse cell atlases into an integrated landscape.
\newblock {\em bioRxiv}, page 397042, 2018.

\bibitem[PSF19]{PSF19}
Xavier Pennec, Stefan Sommer, and Tom Fletcher.
\newblock {\em Riemannian Geometric Statistics in Medical Image Analysis}.
\newblock Academic Press, 2019.

\bibitem[PVG{\etalchar{+}}11]{SKLEARN}
Fabian Pedregosa, Ga{\"e}l Varoquaux, Alexandre Gramfort, Vincent Michel,
  Bertrand Thirion, Olivier Grisel, Mathieu Blondel, Peter Prettenhofer, Ron
  Weiss, Vincent Dubourg, et~al.
\newblock Scikit-learn: Machine learning in python.
\newblock {\em Journal of machine learning research}, 12(Oct):2825--2830, 2011.

\bibitem[Sap06]{SAPORTA}
Gilbert Saporta.
\newblock {\em Probabilit{\'e}s, analyse des donn{\'e}es et statistique}.
\newblock Editions Technip, 2006.

\bibitem[SMB{\etalchar{+}}20]{SITZMANN}
Vincent Sitzmann, Julien Martel, Alexander Bergman, David Lindell, and Gordon
  Wetzstein.
\newblock Implicit neural representations with periodic activation functions.
\newblock {\em Advances in Neural Information Processing Systems}, 33, 2020.

\bibitem[TAC{\etalchar{+}}20]{BatchBenchmark}
Hoa Thi~Nhu Tran, Kok~Siong Ang, Marion Chevrier, Xiaomeng Zhang, Nicole
  Yee~Shin Lee, Michelle Goh, and Jinmiao Chen.
\newblock A benchmark of batch-effect correction methods for single-cell rna
  sequencing data.
\newblock {\em Genome Biology}, 21(1):1--32, 2020.

\bibitem[Tri10]{Triebel}
Hans Triebel.
\newblock {\em Bases in function spaces, sampling, discrepancy, numerical
  integration}, volume~11.
\newblock European Mathematical Society, 2010.

\bibitem[Vil08]{Villani_OT}
C{\'e}dric Villani.
\newblock {\em Optimal transport: old and new}, volume 338.
\newblock Springer Science \& Business Media, 2008.

\bibitem[VLL{\etalchar{+}}10]{AE2}
Pascal Vincent, Hugo Larochelle, Isabelle Lajoie, Yoshua Bengio, and
  Pierre-Antoine Manzagol.
\newblock Stacked denoising autoencoders: Learning useful representations in a
  deep network with a local denoising criterion.
\newblock {\em Journal of machine learning research}, 11(Dec):3371--3408, 2010.

\bibitem[WSL17]{CVAE2}
Liwei Wang, Alexander Schwing, and Svetlana Lazebnik.
\newblock Diverse and accurate image description using a variational
  auto-encoder with an additive gaussian encoding space.
\newblock In {\em Advances in Neural Information Processing Systems}, pages
  5756--5766, 2017.

\bibitem[XD18]{XD18}
Jiacheng Xu and Greg Durrett.
\newblock Spherical latent spaces for stable variational autoencoders.
\newblock {\em arXiv preprint arXiv:1808.10805}, 2018.

\bibitem[ZPIE17]{CYCLE_GAN}
Jun-Yan Zhu, Taesung Park, Phillip Isola, and Alexei~A Efros.
\newblock Unpaired image-to-image translation using cycle-consistent
  adversarial networks.
\newblock In {\em Proceedings of the IEEE international conference on computer
  vision}, pages 2223--2232, 2017.

\end{thebibliography}
